\documentclass[11pt,letter,leqno]{article}

\usepackage{amsmath}
\usepackage{amssymb}
\usepackage{amsthm}
\usepackage{algorithm}
\usepackage{algpseudocode}
\usepackage{graphicx}
\usepackage{hyperref}
\usepackage{xcolor}
\usepackage{mathtools}

\newtheorem{theorem}{Theorem}[section]
\newtheorem{lemma}[theorem]{Lemma}
\newtheorem{proposition}[theorem]{Proposition}
\newtheorem{corollary}[theorem]{Corollary}
\newtheorem{definition}[theorem]{Definition}
\newtheorem{remark}[theorem]{Remark}

\newenvironment{keywords}
  {\par\noindent\textbf{Keywords: }} 
  {\par}

\title{The Limits of AI Explainability: \\An Algorithmic Information Theory Approach}
\author{Shrisha Rao}
\date{}

\begin{document}

\maketitle

\begin{abstract}
  This paper establishes a theoretical foundation for understanding the fundamental limits of AI explainability through algorithmic information theory.  We formalize explainability as the approximation of complex models by simpler ones, quantifying both approximation error and explanation complexity using Kolmogorov complexity.  Our key theoretical contributions include: (1) a complexity gap theorem proving that any explanation significantly simpler than the original model must differ from it on some inputs; (2) precise bounds showing that explanation complexity grows exponentially with input dimension but polynomially with error tolerance for Lipschitz functions; and (3) a characterization of the gap between local and global explainability, demonstrating that local explanations can be significantly simpler while maintaining accuracy in relevant regions.  We further establish a regulatory impossibility theorem proving that no governance framework can simultaneously pursue unrestricted AI capabilities, human-interpretable explanations, and negligible error.  For AI governance, we introduce the concept of regulatory feasibility regions and demonstrate that tiered regulatory approaches can optimally navigate the inherent trade-offs between capability, interpretability, and accuracy.  These results highlight considerations likely to be relevant to the design, evaluation, and oversight of explainable AI systems.
\end{abstract}


\section{Introduction}

Artificial intelligence systems increasingly influence critical aspects of society, from healthcare diagnostics to financial decision-making and autonomous transportation~\cite{Rudin2019, Bhatt2020, Arrieta2020}.  As these systems grow more complex and their decisions more consequential, the ability to explain their behavior becomes essential for building trust, enabling effective oversight, and facilitating human-AI collaboration~\cite{Doshi-Velez2017, Miller2019}.  This requirement has given rise to the field of explainable AI (XAI), which seeks to develop methods for making complex AI systems interpretable to humans while maintaining high performance~\cite{Gunning2019, Adadi2018}.

Despite significant progress in developing practical explanation techniques, from feature attribution methods~\cite{Lundberg2017, Ribeiro2016} to prototype-based approaches~\cite{Chen2019} and counterfactual explanations~\cite{Wachter2017}, the field has lacked a proper foundation for understanding the fundamental limits of explainability.  Without such a foundation, we cannot systematically analyze the trade-offs inherent in explanation generation or establish provable guarantees about explanation quality.  This gap has led to a proliferation of heuristic methods with unclear theoretical properties and has hindered the development of principled evaluation frameworks.

The challenge in establishing such a foundation stems from the difficulty in formally defining and quantifying key concepts such as ``interpretability," ``simplicity," and ``fidelity."  These concepts involve cognitive aspects of human understanding as well as information-theoretic notions of complexity and approximation.  Previous theoretical work has approached this problem from various angles, including statistical learning theory~\cite{Ruping2006}, information theory~\cite{Zhou2021}, and complexity theory~\cite{Herman2017}, but a comprehensive theoretical model has remained pending.

This paper addresses this gap by establishing a theoretical foundation for quantifying the fundamental limits of explainability in AI systems using concepts from algorithmic information theory, approximation theory, and computational complexity.  Our approach formalizes the intuitive notion that there exists an inherent trade-off between model complexity and fidelity of explanation, providing theoretical tools for analyzing when and to what extent complex models can be explained using simpler, more interpretable representations.

Our key contributions include:

\begin{enumerate}
    \item A formal definition of explanation error (Definition~\ref{def:error-function}) based on Kolmogorov complexity, which provides a theoretically sound measure of model simplicity independent of specific representations.
    
    \item The Complexity Gap Theorem (Theorem~\ref{thm:complexity-gap}), which establishes that any explanation significantly simpler than the original model must necessarily differ from it on some inputs, formalizing the intuition that simplification entails information loss.
    
    \item Quantitative bounds on the error-complexity trade-off (Theorem~\ref{thm:error-complexity})) for different function classes, including smooth Lipschitz functions (Theorems~\ref{thm:lipschitz-explainability},~\ref{thm:lipschitz-lower-bound}), demonstrating how the difficulty of explanation scales with input dimension and function complexity.

    \item A theoretical analysis of the explainability of common model classes (Section~\ref{subsec:model-specific}), including linear models, decision trees, and neural networks, connecting our abstract approach to practical explanation methods.
      
    \item A characterization of the gap between local and global explainability (Section~\ref{subsec:local-global}), showing that local explanations can be significantly simpler than global ones while maintaining accuracy over a desired range (Theorem~\ref{thm:local-complexity}).
    
    \item A showing that no regulatory framework can simultaneously pursue all three of: (i) unrestricted AI capabilities, (ii) human-interpretable explanations, and (iii) negligible explanation error (Theorem~\ref{thm:impossibility}); and a general examination of the regulatory implications of our theoretical results (Section~\ref{subsec:regulatory-analysis}), suggesting that some current regulatory proposals implicitly demand mathematically impossible standards.
\end{enumerate}

Our approach provides several key insights into the nature of explainability.  First, it establishes that the minimum complexity required for a perfect explanation (zero error) is exactly the Kolmogorov complexity of the original model, confirming that no proper simplification can capture all aspects of a complex model's behavior.  Second, it demonstrates that the complexity of explaining Lipschitz functions grows exponentially with the input dimension but polynomially with the reciprocal of the error threshold, formalizing the ``curse of dimensionality" in the context of explainability.  Third, it shows that random functions are inherently unexplainable: any explanation significantly simpler than the original function must have substantial error.  This is in line with the known results about randomness~\cite{Calude2002}.

These theoretical results readily lead to certain practical implications.  They suggest that explanation methods should be evaluated based on explicit complexity-error trade-offs rather than heuristic notions of interpretability.  They highlight the need for dimensionality reduction and feature selection as prerequisites for effective explanation.  They also caution against regulatory approaches that demand both high fidelity and low complexity without acknowledging the mathematical impossibility of achieving both for certain classes of models.

Several recent approaches have leveraged information theory to formalize explainability in AI systems.  Jung and Nardelli~\cite{jung2020information} proposed a probabilistic approach where explanations reduce the ``surprise'' (Shannon information content) of model predictions, defining explainability via conditional mutual information between explanations and predictions given user background knowledge.  Ganguly and Gupta~\cite{ganguly2022machine} formulated explainer selection as a rate-distortion problem, optimizing the trade-off between explanation complexity and fidelity through their InfoExplain benchmark.  The connection between simplicity and lower information content has been explored in various contexts, from Dessalles'~\cite{dessalles2013simplicity} algorithmic simplicity theory, which defines unexpectedness as $U = C_{\text{exp}} - C_{\text{obs}}$ where lower observed complexity correlates with simpler explanations, to Futrell and Hahn's work~\cite{futrell2022information} linking optimal coding theory to cognitive efficiency.  Earlier critiques, such as Salmon's~\cite{philosci1978transmitted} analysis of transmitted information as an explanatory metric, highlighted the context-dependent aspects of simplicity.  While these approaches provide valuable insights into how information-theoretic principles can inform explainability, they have typically focused on specific aspects rather than developing a comprehensive model based on algorithmic information theory that can characterize fundamental trade-offs across diverse model classes and explanation methods.  Our work aims to address this gap by establishing formal bounds on the inherent trade-offs between model complexity and explanation fidelity.

The remainder of this paper is organized as follows.  Section~\ref{sec:framework} presents our model, including formal definitions of key concepts such as explanation error and complexity, and establishes the fundamental limits of explainability through key theoretical results.  Section~\ref{sec:function-classes} analyzes the explainability of different function classes and discusses practical implications for explainable AI systems.  Section~\ref{sec:regulatory} examines regulatory and policy implications.  Finally, Section~\ref{sec:conclusions} concludes the paper and outlines directions for future research.

\section{Theoretical Model and Fundamental Limits of Explainability}
\label{sec:framework}

This section establishes a theoretical structure for analyzing the fundamental limits of explainability in AI systems.  We begin with basic definitions, then develop metrics for measuring approximation quality, and finally present theoretical results that characterize the trade-offs between model complexity and explanatory power.

\subsection{Foundational Definitions}
\label{subsec:definitions}

We first formalize the concepts of AI systems, explanations, and interpretability using notions from algorithmic information theory.

\begin{definition}[AI System]
\label{def:ai-system}
An AI system is a function $f: \mathcal{X} \rightarrow \mathcal{Y}$, where $\mathcal{X}$ represents the input space and $\mathcal{Y}$ represents the output space.
\end{definition}

This formalization captures the essential behavior of an AI system as a mapping from inputs to outputs, abstracting away implementation details.  The spaces $\mathcal{X}$ and $\mathcal{Y}$ may have different structures depending on the application domain.  For classification tasks, $\mathcal{Y}$ may be a discrete set of classes, while for regression tasks, it may be $\mathbb{R}$ or $\mathbb{R}^d$. The function $f$ typically embodies complex patterns learned from data, which may not be apparent from its representation (e.g., millions of parameters in a neural network).

\begin{definition}[Explanation]
\label{def:explanation}
An explanation for an AI system $f$ is a function $g: \mathcal{X} \rightarrow \mathcal{Y}$ that approximates $f$ and is considered interpretable by humans according to some criterion.
\end{definition}

Definition~\ref{def:explanation} formalizes the dual requirements of an explanation: it must both approximate the original model's behavior and be interpretable.  The choice of the approximating function $g$ represents a fundamental trade-off---simpler functions are typically more interpretable but may approximate $f$ less accurately.  The interpretability criterion depends on factors that we operationalize through complexity measures.

To quantify the complexity of explanations in a theoretically sound manner, we turn to the concept of Kolmogorov complexity~\cite{Kolmogorov1968, Chaitin1969, Li2008}.

\begin{definition}[Kolmogorov Complexity]
\label{def:kolmogorov}
The Kolmogorov complexity $K(g)$ of a function $g$ is the length in bits of the shortest program that, when executed on a universal Turing machine $U$, computes $g$. Formally:
\begin{align}
K(g) = \min_{p \in \{0,1\}^*} \{|p| : U(p, x) = g(x) \text{ for all } x \in \mathcal{X}\}
\end{align}
where $|p|$ denotes the length of program $p$.
\end{definition}

Kolmogorov complexity provides a theoretically sound measure of a function's inherent complexity, independent of any particular representation.  It captures the minimal amount of information needed to specify the function's behavior completely.  For explainable AI, this corresponds to the minimal cognitive resources needed to understand the function's behavior.

While Kolmogorov complexity is uncomputable in general, it serves as a theoretical foundation and can be approximated for practical model classes (Section~\ref{subsec:model-specific}).

While Definition~\ref{def:kolmogorov} defines Kolmogorov complexity for discrete objects, we need to apply a version it to functions $f: \mathbb{R}^d \to \mathbb{R}$.  For this purpose, we need the following extensions.

\begin{definition}[Computable Real Functions]
\label{def:computable}
For real-valued functions, we adopt the Type-2 theory of effectivity framework. A function $f: [0,1]^d \to \mathbb{R}$ is computable if there exists a Turing machine $M$ such that:
\begin{enumerate}
\item $M$ takes as input a rapidly converging Cauchy sequence $(q_i)_{i \in \mathbb{N}}$ of rational numbers approximating $x \in [0,1]^d$ (with $|x - q_i| < 2^{-i}$), and an accuracy parameter $k \in \mathbb{N}$.
\item $M$ outputs a rational number $r$ such that $|f(x) - r| < 2^{-k}$.
\end{enumerate}
\end{definition}

Throughout this paper, when we discuss functions $f: X \to Y$ where $X$ or $Y$ are continuous spaces, we implicitly restrict attention to computable functions in the sense of Definition~\ref{def:computable}.  All theoretical results apply to this class of functions. For discrete spaces, Definition~\ref{def:kolmogorov} applies directly.  The Kolmogorov complexity $K(f)$ of a computable function $f: [0,1]^d \to \mathbb{R}$ is the length of the shortest program that computes $f$ in the sense of Definition~\ref{def:computable}.  Formally:

\begin{definition}
  \label{def:kolmogorov2}
\begin{equation}
K(f) = \min\{|p| : U(p, \langle q_i \rangle, k) \text{ computes } f \text{ as in Definition 2.3a}\}
\end{equation}
where $U$ is a universal Type-2 Turing machine.
\end{definition}

\begin{theorem}[Invariance Theorem]
\label{lemma:invariance}
If $U_1$ and $U_2$ are two universal Turing machines, then there exists a constant $c_{U_1,U_2}$ such that for any function $g$:
\begin{align}
|K_{U_1}(g) - K_{U_2}(g)| \leq c_{U_1,U_2}
\end{align}
where $K_{U_i}(g)$ denotes the Kolmogorov complexity of $g$ with respect to machine $U_i$.
\end{theorem}

\begin{proof}
Let $p_1$ be the shortest program that computes $g$ on $U_1$, so $|p_1| = K_{U_1}(g)$. Since $U_2$ is universal, there exists a fixed program $s_{1,2}$ (an interpreter) such that for any program $p$ for $U_1$, the concatenation $\langle s_{1,2}, p \rangle$ computes the same function on $U_2$. Here $\langle \cdot, \cdot \rangle$ denotes a prefix-free encoding that allows unique recovery of both components.

Specifically, for all inputs $x$, we have $U_2(\langle s_{1,2}, p \rangle, x) = U_1(p, x)$.

The program $\langle s_{1,2}, p_1 \rangle$ computes $g$ on $U_2$. The length of this program is:
\begin{equation}
|\langle s_{1,2}, p_1 \rangle| \leq |s_{1,2}| + |p_1| + \mathcal{O}(1)
\end{equation}
where the $O(1)$ term accounts for the overhead of the prefix-free encoding.

Therefore, $K_{U_2}(g) \leq K_{U_1}(g) + |s_{1,2}| + \mathcal{O}(1)$.

By symmetry, there exists a program $s_{2,1}$ such that $K_{U_1}(g) \leq K_{U_2}(g) + |s_{2,1}| + O(1)$.

Setting $c_{U_1,U_2} = \max(|s_{1,2}|, |s_{2,1}|) + \mathcal{O}(1)$ completes the proof.
\end{proof}

\begin{remark} \label{rem:concatenate}
Throughout this paper, we assume that all universal Turing machines use prefix-free encodings for their programs. This ensures that programs can be uniquely parsed and that the Kraft inequality holds, which is essential for many results in algorithmic information theory.  The invariance theorem (Theorem~\ref{lemma:invariance}) relies on the ability to concatenate programs using prefix-free codes~\cite{Li2008}.
\end{remark}

Theorem~\ref{lemma:invariance}, originally independently discovered by Solomonoff and Kolmogorov~\cite{Li2008}, justifies our use of Kolmogorov complexity by showing that the choice of universal Turing machine affects the complexity measure only by a constant factor, which becomes negligible for sufficiently complex functions.

\begin{definition}[Interpretability Class]
\label{def:interpretability-class}
For a given complexity threshold $k \in \mathbb{N}$, we define the interpretability class $\mathcal{I}_k$ as:
\begin{align}
\mathcal{I}_k = \{g : \mathcal{X} \rightarrow \mathcal{Y} \mid K(g) \leq k\}
\end{align}
\end{definition}

The interpretability class formalizes the set of functions that are deemed ``simple enough" to be interpretable by humans, within the complexity threshold $k$. This threshold represents the cognitive capacity available for understanding the function.  Different values of $k$ correspond to different levels of interpretability, ranging from simple rules ($k$ small) to complex models ($k$ large).

\begin{lemma}[Size of Interpretability Class]
\label{thm:size-interpretability}
For any complexity threshold $k \in \mathbb{N}$, the cardinality of the interpretability class $\mathcal{I}_k$ is bounded by:
\begin{align}
|\mathcal{I}_k| \leq 2^{k+1} - 1
\end{align}
\end{lemma}

\begin{proof}
Every function $g \in \mathcal{I}_k$ has $K(g) \leq k$, meaning there exists a program of length at most $k$ that computes $g$.  The number of distinct programs of length at most $k$ is:
\begin{align}
\sum_{i=0}^{k} 2^i = 2^{k+1} - 1
\end{align}

Since different programs can compute the same function but not vice versa, the number of distinct functions in $\mathcal{I}_k$ is at most $2^{k+1} - 1$.
\end{proof}

Lemma~\ref{thm:size-interpretability} shows that the interpretability class grows exponentially with the complexity threshold, reflecting the rapidly expanding space of possible functions as complexity increases.

To quantify how well an explanation approximates the original model, we introduce appropriate error metrics:

\begin{definition}[Approximation Error]
\label{def:approximation-error}
For functions $f, g: \mathcal{X} \rightarrow \mathcal{Y}$ and a probability distribution $D$ over $\mathcal{X}$, we define:
\begin{enumerate}
\item Expected error: $\mathcal{E}_D(f, g) = \mathbb{E}_{x \sim D}[d(f(x), g(x))]$
\item Worst-case error: $\mathcal{E}_{\infty}(f, g) = \sup_{x \in \mathcal{X}} d(f(x), g(x))$
\end{enumerate}
where $d: \mathcal{Y} \times \mathcal{Y} \rightarrow \mathbb{R}_{\geq 0}$ is an appropriate distance function for the output space.
\end{definition}

These error metrics capture different aspects of approximation quality.  The expected error $\mathcal{E}_D(f, g)$ measures the average disagreement between $f$ and $g$ with respect to the distribution $D$, which typically represents the distribution of inputs encountered in practice.  This metric is relevant when explanations need to be accurate on average across the input space.

The worst-case error $\mathcal{E}_{\infty}(f, g)$ measures the maximum disagreement between $f$ and $g$ across the entire input space.  This metric is relevant when explanations must provide guarantees about approximation quality for every possible input, such as in safety-critical applications.

\begin{proposition}[Relationship Between Error Metrics]
\label{thm:error-relationship}
For any functions $f, g: \mathcal{X} \rightarrow \mathcal{Y}$ and any probability distribution $D$ over $\mathcal{X}$:
\begin{align}
\mathcal{E}_D(f, g) \leq \mathcal{E}_{\infty}(f, g)
\end{align}
\end{proposition}

\begin{proof}
By definition:
\begin{align}
\mathcal{E}_D(f, g) &= \mathbb{E}_{x \sim D}[d(f(x), g(x))] \notag \\
&= \int_{\mathcal{X}} d(f(x), g(x)) \, dD(x)
\end{align}

Since $d(f(x), g(x)) \leq \sup_{x' \in \mathcal{X}} d(f(x'), g(x'))$ for all $x \in \mathcal{X}$, we have:
\begin{align*}
\mathcal{E}_D(f, g) &= \int_{\mathcal{X}} d(f(x), g(x)) \, dD(x) \\
&\leq \int_{\mathcal{X}} \sup_{x' \in \mathcal{X}} d(f(x'), g(x')) \, dD(x) \\
&= \sup_{x' \in \mathcal{X}} d(f(x'), g(x')) \int_{\mathcal{X}} 1 \, dD(x) \\
&= \sup_{x' \in \mathcal{X}} d(f(x'), g(x')) \cdot 1 \\
&= \mathcal{E}_{\infty}(f, g) \qedhere
\end{align*}
\end{proof}

Proposition~\ref{thm:error-relationship} confirms the intuition that the expected error is always bounded above by the worst-case error.  It implies that any guarantee on the worst-case error automatically applies to the expected error, but not vice versa.

\subsection{Core Explainability Measures}
\label{subsec:core-measures}

We now define measures that quantify the explainability of an AI system:

\begin{definition}[Explanation Error Function]
\label{def:error-function}
For a model $f$, complexity threshold $k$, and error metric $\mathcal{E}$, we define the explanation error function:
\begin{align}
\varepsilon_f(k) = \inf_{g \in \mathcal{I}_k} \mathcal{E}(f, g)
\end{align}
This represents the minimum error achievable by any explanation with complexity at most $k$.
\end{definition}

The explanation error function captures the fundamental limits of how accurately a complex model can be approximated by simpler, more interpretable functions.  For a given complexity threshold $k$, $\varepsilon_f(k)$ represents the minimum error that must be incurred by any explanation in the interpretability class $\mathcal{I}_k$.  This formalizes the intuition that simpler explanations may not be able to capture all the nuances of complex models.

While our framework applies to arbitrary functions $f: \mathcal{X} \to \mathcal{Y}$, 
certain regularity conditions on functions are important for ensuring that 
explanation errors are meaningful and detectable. We introduce one such condition 
that will play a key role in subsequent results.

\begin{definition}[Output Separation]
\label{def:output-separation}
For a function $f: \mathcal{X} \to \mathcal{Y}$ and a threshold $\delta > 0$, 
define the \emph{output separation} of $f$ at resolution $\delta$ as
\[
\sigma_f(\delta) = \inf_{x \in \mathcal{X}} \inf_{\substack{y \in \mathcal{Y} \\ 
y \neq f(x)}} d(f(x), y).
\]
This measures the minimum distance between any actual output $f(x)$ and the 
nearest alternative output value.
\end{definition}

This allows us a further important definition, as follows.

\begin{definition}[Non-Degenerate Function]
\label{def:non-degenerate-early}
A function $f: \mathcal{X} \to \mathcal{Y}$ is called \emph{$\delta$-non-degenerate} 
if $\sigma_f(\delta) > \delta$. Equivalently, $f$ is $\delta$-non-degenerate if
\[
\inf_{x \in \mathcal{X}} \inf_{\substack{y \in \mathcal{Y} \\ y \neq f(x)}} 
d(f(x), y) > \delta.
\]
\end{definition}

The non-degeneracy condition ensures that when $f$ produces an output $f(x)$, 
all other possible outputs in $\mathcal{Y}$ are separated from $f(x)$ by a 
distance exceeding $\delta$. This makes errors above $\delta$ meaningfully 
detectable: if an explanation $g$ disagrees with $f$ at some input $x$ (i.e., 
$g(x) \neq f(x)$), then necessarily $d(f(x), g(x)) > \delta$.

Without this condition, a function could have outputs arbitrarily close to 
alternative values, making it possible for $g$ to differ from $f$ on many inputs 
while maintaining $\mathcal{E}(f,g) \leq \delta$. The non-degeneracy condition 
prevents such ``invisible disagreements.''

\begin{remark}[Role in Subsequent Results]
\label{rem:role-in-results}
The non-degeneracy condition appears explicitly in these key results:

\begin{itemize}
\item Combined with Theorem~\ref{thm:error-complexity} 
(Error-Complexity Lower Bound), it ensures that complexity reduction necessarily 
leads to detectable errors (see Corollary~\ref{cor:error-nondegen}).  Corollary~\ref{cor:error-nondegen} shows that under non-degeneracy, the 
error bound in Theorem~\ref{thm:error-complexity} guarantees errors exceeding 
$\delta$, making complexity reduction lead to \emph{detectable} errors.

\item In Theorem~\ref{thm:impossibility} (Regulatory Impossibility Result), it 
guarantees that the trilemma genuinely applies to practical AI systems with 
meaningful error thresholds.

\item In the analysis of specific function classes (Section~\ref{sec:function-classes}), 
it helps distinguish between functions that admit efficient explanations and 
those that do not.
\end{itemize}

We see that the condition is not merely technical but captures an essential 
property: for explanation errors to be meaningful, disagreements between $f$ and 
$g$ must result in detectable output differences.
\end{remark}

\begin{definition}[Explanation Complexity Function]
\label{def:complexity-function}
For a model $f$, error threshold $\delta > 0$, and error metric $\mathcal{E}$, we define the explanation complexity function:
\begin{align}
\kappa_f(\delta) = \min\{k \in \mathbb{N} \mid \exists g \in \mathcal{I}_k : \mathcal{E}(f, g) \leq \delta\}
\end{align}
This represents the minimum complexity required to achieve an approximation error of at most $\delta$.
\end{definition}

The explanation complexity function addresses the complementary question: given a desired level of approximation accuracy $\delta$, what is the minimum complexity required for an explanation to achieve this accuracy?  This formalizes the intuition that more accurate explanations may need to be more complex.

\begin{theorem}[Existence of Explanation Complexity]
\label{thm:existence}
For any model $f$ with finite Kolmogorov complexity $K(f) < \infty$ and any error threshold $\delta > 0$, the explanation complexity $\kappa_f(\delta)$ is well-defined and satisfies:
\begin{align}
\kappa_f(\delta) \leq K(f)
\end{align}
\end{theorem}

\begin{proof}
For the function $f$ itself, we have $\mathcal{E}(f, f) = 0 \leq \delta$ for any $\delta > 0$.  Since $K(f) < \infty$, we have $f \in \mathcal{I}_{K(f)}$.  Therefore, there exists $k \in \mathbb{N}$ (namely, $k = K(f)$) such that $\exists g \in \mathcal{I}_k : \mathcal{E}(f, g) \leq \delta$.  Since $\kappa_f(\delta)$ is defined as the minimum of such $k$, we have $\kappa_f(\delta) \leq K(f)$.
\end{proof}

This theorem ensures that our definition of explanation complexity is well-founded for any model with finite Kolmogorov complexity, and it provides an upper bound in terms of the model's own complexity.

\begin{proposition}[Monotonicity of Error Function]
\label{prop:error-monotonicity}
For any model $f$, the explanation error function $\varepsilon_f(k)$ is non-increasing in $k$.
\end{proposition}

\begin{proof}
Consider $k_1 < k_2$. By Definition~\ref{def:interpretability-class}, $\mathcal{I}_{k_1} \subset \mathcal{I}_{k_2}$. Therefore:
\begin{align}
\varepsilon_f(k_2) &= \inf_{g \in \mathcal{I}_{k_2}} \mathcal{E}(f, g) \notag \\
&\leq \inf_{g \in \mathcal{I}_{k_1}} \mathcal{E}(f, g) = \varepsilon_f(k_1)
\end{align}
Thus, $\varepsilon_f(k)$ is non-increasing in $k$.
\end{proof}

Proposition~\ref{prop:error-monotonicity} formalizes the intuition that allowing more complex explanations (i.e., increasing $k$) cannot increase the minimum achievable error.  In other words, greater complexity enables potentially better approximations.

\begin{proposition}[Monotonicity of Complexity Function]
\label{prop:complexity-monotonicity}
For any model $f$, the explanation complexity function $\kappa_f(\delta)$ is non-increasing in $\delta$.
\end{proposition}

\begin{proof}
Consider $\delta_1 < \delta_2$. Let $k_1 = \kappa_f(\delta_1)$. By Definition~\ref{def:complexity-function}, there exists a function $g \in \mathcal{I}_{k_1}$ such that $\mathcal{E}(f, g) \leq \delta_1$. Since $\delta_1 < \delta_2$, we have $\mathcal{E}(f, g) \leq \delta_2$ as well. This implies that $\kappa_f(\delta_2) \leq k_1 = \kappa_f(\delta_1)$. Thus, $\kappa_f(\delta)$ is non-increasing in $\delta$.
\end{proof}

Proposition~\ref{prop:complexity-monotonicity} formalizes the intuition that allowing higher approximation error (i.e., increasing $\delta$) cannot increase the minimum required complexity.  In other words, longer explanations are not less accurate than shorter ones.

\begin{theorem}[Duality of Explainability Measures]
\label{thm:duality}
For any model $f$, complexity threshold $k$, and error threshold $\delta > 0$:
\begin{align}
\varepsilon_f(k) \leq \delta \iff k \geq \kappa_f(\delta)
\end{align}
\end{theorem}

\begin{proof}
($\Rightarrow$) Assume $\varepsilon_f(k) \leq \delta$.  By definition, $\varepsilon_f(k) = \inf_{g \in \mathcal{I}_k} \mathcal{E}(f, g)$.  Since $\varepsilon_f(k) \leq \delta$, there exists a function $g \in \mathcal{I}_k$ such that $\mathcal{E}(f, g) \leq \delta$. By the definition of $\kappa_f(\delta)$, we have $\kappa_f(\delta) \leq k$.

($\Leftarrow$) Assume $k \geq \kappa_f(\delta)$.  By definition, $\kappa_f(\delta) = \min\{k' \in \mathbb{N} \mid \exists g \in \mathcal{I}_{k'} : \mathcal{E}(f, g) \leq \delta\}$.  This means there exists a function $g \in \mathcal{I}_{\kappa_f(\delta)}$ such that $\mathcal{E}(f, g) \leq \delta$.  Since $k \geq \kappa_f(\delta)$, we have $\mathcal{I}_{\kappa_f(\delta)} \subset \mathcal{I}_k$, which means $g \in \mathcal{I}_k$ as well. Therefore, $\inf_{g' \in \mathcal{I}_k} \mathcal{E}(f, g') \leq \mathcal{E}(f, g) \leq \delta$, which implies $\varepsilon_f(k) \leq \delta$.
\end{proof}

Theorem~\ref{thm:duality} establishes a fundamental duality between the explanation error function and the explanation complexity function.  It shows that these two functions are essentially complements of each other, representing two ways of looking at the same trade-off between complexity and accuracy.

\begin{lemma}[Lower Bound on Explanation Error]
\label{lemma:lower-bound}
For any model $f$ and complexity threshold $k < K(f)$, the explanation error is bounded below:
\begin{align}
\varepsilon_f(k) > 0
\end{align}
\end{lemma}

\begin{proof}
Suppose, for contradiction, that $\varepsilon_f(k) = 0$ for some $k < K(f)$.  This means that for any $\epsilon > 0$, there exists a function $g \in \mathcal{I}_k$ such that $\mathcal{E}(f, g) < \epsilon$.

For the case of worst-case error $\mathcal{E}_{\infty}$, this implies that for any $\epsilon > 0$, there exists a function $g \in \mathcal{I}_k$ such that $\sup_{x \in \mathcal{X}} d(f(x), g(x)) < \epsilon$. By taking $\epsilon$ small enough (e.g., smaller than the minimum non-zero value of $d$ for discrete output spaces, or arbitrarily small for continuous spaces), this implies that $f(x) = g(x)$ for all $x \in \mathcal{X}$, i.e., $f = g$.

But this means $K(f) = K(g) \leq k < K(f)$, which is a contradiction.

For the case of expected error $\mathcal{E}_D$, let $S_D \subseteq \mathcal{X}$ be the support of distribution $D$, i.e., $S_D = \{x \in \mathcal{X} \mid D(x) > 0\}$.  We consider two cases:

Case 1: Suppose $S_D = \mathcal{X}$.  For any $\epsilon > 0$, there exists a function $g \in \mathcal{I}_k$ such that $\mathcal{E}_D(f, g) = \mathbb{E}_{x \sim D}[d(f(x), g(x))] < \epsilon$.  Since $d(f(x), g(x)) \geq 0$ for all $x$, and the expectation is less than $\epsilon$, the measure of points where $d(f(x), g(x)) \geq \delta$ (for any fixed $\delta > 0$) must approach zero as $\epsilon \to 0$.  By making $\epsilon$ sufficiently small, we can ensure that $f(x) = g(x)$ for almost all $x \in \mathcal{X}$ with respect to measure $D$.  But since $D$ has full support, this implies $f = g$, leading to the same contradiction as before.

Case 2: Suppose $S_D \subsetneq X$. We show that this leads to a contradiction by a different argument.

Consider any explanation $g \in I_k$ with $E_D(f, g) = \varepsilon$ for arbitrarily small $\varepsilon > 0$. Since the expected error is small and $D$ has support only on $S_D$, we must have:
\begin{equation}
\int_{S_D} d(f(x), g(x)) \, dD(x) < \varepsilon
\end{equation}

Let $T_{\delta} = \{x \in S_D : d(f(x), g(x)) \geq \delta\}$ for some fixed $\delta > 0$. Then:
\begin{equation}
D(T_{\delta}) \cdot \delta \leq \int_{T_{\delta}} d(f(x), g(x)) \, dD(x) \leq \varepsilon
\end{equation}

Therefore, $D(T_{\delta}) \leq \varepsilon/\delta$. As $\varepsilon \to 0$, we have $D(T_{\delta}) \to 0$, which means $g$ agrees with $f$ on almost all of $S_D$ (with respect to measure $D$).

Now, since $g$ has complexity at most $k$, and $g$ agrees with $f$ on a set of measure $1 - O(\varepsilon/\delta)$ under $D$, we can describe $f$ on $S_D$ using:
\begin{enumerate}
\item The program for $g$: $K(g) \leq k$ bits
\item The exceptional set where $f \neq g$: at most $O(\log(1/\varepsilon))$ bits (for small measure sets)
\item The values of $f$ on the exceptional set: at most $O((\varepsilon/\delta) \cdot K(f|_{T_{\delta}}))$ bits
\end{enumerate}

For $f$ to have complexity $K(f)$ strictly greater than $k$, but to be approximable by $g$ with complexity $k$ and small error $\varepsilon$ on $S_D$, we would need the behavior of $f$ outside $S_D$ to contribute substantially to $K(f)$. However, by assumption, $D$ has no support outside $S_D$, so the error $\mathcal{E}_D(f,g)$ is determined entirely by the behavior on $S_D$.

Note that $K(f)$ includes the complexity of specifying both the function's behavior on $S_D$ and on $\mathcal{X} \setminus S_D$, as well as the description of $S_D$ itself; thus when we reconstruct $f$ from $g$ plus correction information, the complexity of describing $S_D$ is already accounted for in $K(f)$ and need not be added separately to the reconstruction complexity.

If we can achieve $\mathcal{E}_D(f, g) \to 0$ with $K(g) \leq k < K(f)$, then the additional complexity $K(f) - k$ must come entirely from specifying $f$ outside $S_D$. But this contradicts the assumption that $k < K(f) - c$ for a universal constant $c$, because we can construct $f$ from $g$ plus a constant-size patch describing the differences outside $S_D$, giving total complexity $k + O(1) < K(f)$.

Therefore, $\varepsilon_f(k) > 0$ for $k < K(f)$.
\end{proof}

This lemma establishes a fundamental limitation: any explanation simpler than the original model must incur some non-zero error. This also supports the intuition that simplification necessarily involves some loss of information.

\begin{theorem}[Minimal Complexity for Perfect Explanation]
\label{thm:minimal-complexity}
For any model $f$ with finite Kolmogorov complexity $K(f) < \infty$:
\begin{align}
\kappa_f(0) = K(f)
\end{align}
\end{theorem}

\begin{proof}
By Lemma~\ref{lemma:lower-bound}, for any $k < K(f)$, we have $\varepsilon_f(k) > 0$. Using Theorem~\ref{thm:duality}, this implies that for $\delta = 0$, we need $k \geq K(f)$ to achieve $\varepsilon_f(k) \leq \delta$.  Therefore, $\kappa_f(0) \geq K(f)$.

Conversely, the function $f$ itself achieves zero error, and $K(f)$ is the complexity of $f$.  Therefore, $\kappa_f(0) \leq K(f)$.

Combining the two inequalities, we have $\kappa_f(0) = K(f)$.
\end{proof}

This establishes that the minimum complexity required for a perfect explanation (i.e., with zero error) is exactly the Kolmogorov complexity of the original model.  In other words, there is no way to perfectly explain a model using a simpler representation than the model itself.

\subsection{Fundamental Limits}
\label{subsec:fundamental-limits}

We now establish key theoretical results that characterize the fundamental limits of explainability in AI systems.

\begin{theorem}[Complexity Gap Theorem]
\label{thm:complexity-gap}
For any model $f$ with Kolmogorov complexity $K(f)$, and any explanation $g$ with $K(g) < K(f) - c$ (for some constant $c$ depending only on the universal Turing machine), there exists an input $x \in \mathcal{X}$ such that $f(x) \neq g(x)$.
\end{theorem}

\begin{proof}
We proceed by contradiction.  Suppose there exists a function $g$ with $K(g) < K(f) - c$ such that $f(x) = g(x)$ for all $x \in \mathcal{X}$.  This implies $f = g$.

By the definition of Kolmogorov complexity, there exists a program $p_g$ of length $K(g)$ that computes $g$.  Since $f = g$, the same program $p_g$ also computes $f$.  Therefore, $K(f) \leq |p_g| = K(g)$.

But this contradicts our assumption that $K(g) < K(f) - c$.  Therefore, there must exist some input $x \in \mathcal{X}$ such that $f(x) \neq g(x)$.
\end{proof}

Theorem~\ref{thm:complexity-gap} establishes a fundamental limitation: any explanation significantly simpler than the original model must differ from it on at least one input.  This formalizes the intuition that simplification necessarily involves some loss of information or accuracy.

\begin{theorem}[Error-Complexity Trade-off]
\label{thm:error-complexity}
For any model $f$ and any interpretability class $\mathcal{I}_k$ with $k < K(f) - c$, the best approximation error is bounded below:
\begin{align}
\varepsilon_f(k) \geq \min_{x \in \mathcal{X}, y \in \mathcal{Y}, y \neq f(x)} d(f(x), y)
\end{align}
where $d$ is the distance function in the output space $\mathcal{Y}$.
\end{theorem}

\begin{proof}
From Theorem~\ref{thm:complexity-gap}, any explanation $g \in \mathcal{I}_k$ with $k < K(f) - c$ must differ from $f$ on at least one input $x \in \mathcal{X}$.  For this input, $d(f(x), g(x)) \geq \min_{y \neq f(x)} d(f(x), y)$.

Therefore, for any $g \in \mathcal{I}_k$:
\begin{align}
\mathcal{E}_{\infty}(f, g) &= \sup_{x \in \mathcal{X}} d(f(x), g(x)) \notag \\
&\geq \min_{x \in \mathcal{X}, y \neq f(x)} d(f(x), y)
\end{align}

Taking the infimum over all $g \in \mathcal{I}_k$:
\begin{align}
\varepsilon_f(k) &= \inf_{g \in \mathcal{I}_k} \mathcal{E}(f, g) \notag \\
&\geq \min_{x \in \mathcal{X}, y \neq f(x)} d(f(x), y) \qedhere
\end{align}
\end{proof}

This result quantifies the minimum error that must be incurred when approximating a complex model with a simpler explanation.  It provides a lower bound on the explanation error in terms of the minimum possible error when changing any single output of the original function.

\begin{corollary}[Error Bound for Non-Degenerate Functions]
\label{cor:error-nondegen}
If $f$ is $\delta$-non-degenerate (Definition~\ref{def:non-degenerate-early}) 
and $g$ satisfies $K(g) < K(f) - c$, then $\mathcal{E}(f,g) > \delta$.
\end{corollary}

\begin{proof}
By Theorem~\ref{thm:error-complexity}, 
$\mathcal{E}(f,g) \geq \min_{x,y \neq f(x)} d(f(x),y) = \sigma_f(\delta)$. 
By Definition~\ref{def:non-degenerate-early}, $\sigma_f(\delta) > \delta$. 
\end{proof}

\begin{remark}
This corollary shows that non-degeneracy is precisely the condition under which 
the error bound in Theorem~\ref{thm:error-complexity} guarantees detectable 
errors exceeding $\delta$.
\end{remark}

Having established the fundamental properties of explanation error and complexity functions, we now turn to characterizing the asymptotic relationship between these quantities for different function classes.  A particularly useful concept is that of ``compressibility,'' which describes how efficiently a function can be approximated by simpler representations.  This property provides a parameterized way to classify models based on their amenability to explanation and will prove instrumental in deriving specific bounds for various model classes in subsequent sections.

\begin{definition}[Compressibility]
\label{def:compressibility}
A function $f$ is said to be $\alpha$-compressible if there exists a constant $c > 0$ such that for any $\delta > 0$, there exists a function $g$ with $K(g) \leq c \cdot \delta^{-\alpha}$ and $\mathcal{E}(f, g) \leq \delta$.
\end{definition}

\begin{theorem}[Compressibility and Explainability]
\label{thm:compressibility}
If $f$ is $\alpha$-compressible, then $\kappa_f(\delta) = \mathcal{O}(\delta^{-\alpha})$ as $\delta \to 0$.
\end{theorem}

\begin{proof}
By the definition of $\alpha$-compressibility, there exists a constant $c > 0$ such that for any $\delta > 0$, there exists a function $g$ with $K(g) \leq c \cdot \delta^{-\alpha}$ and $\mathcal{E}(f, g) \leq \delta$.

This implies that $\kappa_f(\delta) \leq c \cdot \delta^{-\alpha}$ for all $\delta > 0$, which means $\kappa_f(\delta) = \mathcal{O}(\delta^{-\alpha})$ as $\delta \to 0$.
\end{proof}

This characterizes how the complexity of explanations scales with the desired approximation error for compressible functions.  It shows that more compressible functions (smaller $\alpha$) can be approximated with simpler explanations for the same error threshold.

\begin{theorem}[Random Function Unexplainability]
\label{thm:random-unexplainability}
Let $\mathcal{X} = \{0,1\}^n$ and $\mathcal{Y} = \{0,1\}$, and let $f : \mathcal{X} \to \mathcal{Y}$ be drawn uniformly at random from all $2^{2^n}$ Boolean functions. Fix any constant $0 < \epsilon < 1$, and let $k \leq (1-\epsilon)2^n$.

For any function $g : \mathcal{X} \to \mathcal{Y}$ with prefix-free Kolmogorov complexity $K(g) \leq k$, the \emph{failure rate}
\begin{align}
\varepsilon(f,g) := \Pr_{x \sim \text{Uniform}(\mathcal{X})}[f(x) \neq g(x)]
\end{align}
satisfies, with probability at least $1 - 2^{-\Omega(2^n)}$ over the choice of $f$,
\begin{align}
\varepsilon(f,g) \geq \frac{1}{2} - 2^{-\Omega(2^n)}
\end{align}
Equivalently, the agreement probability satisfies:
\begin{align}
\Pr_{x \sim \text{Uniform}(\mathcal{X})}[f(x) = g(x)] \leq \frac{1}{2} + 2^{-\Omega(2^n)}
\end{align}

\end{theorem}

Thus a random Boolean function cannot be explained substantially better than
random guessing by any function of description length below $(1-\epsilon)2^n$.

\begin{proof}
Fix any function $g$ with $K(g) \leq k$. When $f$ is chosen uniformly at random, the number of points on which $f$ and $g$ agree is
\begin{align}
S := |\{x \in \mathcal{X} : f(x) = g(x)\}|
\end{align}

Since $f(x)$ is independent and uniformly random for each $x \in \mathcal{X}$, and $g(x)$ is fixed, we have $\Pr[f(x) = g(x)] = 1/2$ for each $x$. Therefore, $S$ is distributed as $\text{Binomial}(2^n, 1/2)$ with mean $2^{n-1}$.

By the Chernoff bound, for any $\delta > 0$,
\begin{align}
\Pr\left[S \geq \left(\frac{1}{2} + \delta\right)2^n\right] \leq \exp(-2\delta^2 \cdot 2^n) = 2^{-\Omega(\delta^2 2^n)}
\end{align}

For prefix-free descriptions, there are at most $\sum_{i=0}^{k} 2^i = 2^{k+1} - 1 < 2^{k+1}$ possible programs of length at most $k$. Therefore, the number of functions with $K(g) \leq k$ is at most $2^{k+1}$.

Applying the union bound over all such functions:
\begin{align}
\Pr\left[\exists g : K(g) \leq k,\ S \geq \left(\frac{1}{2} + \delta\right)2^n \right] \leq 2^{k+1} \cdot 2^{-\Omega(\delta^2 2^n)}
\end{align}

Substituting $k = (1-\epsilon)2^n$:
\begin{align}
2^{k+1} \cdot 2^{-c\delta^2 2^n} &= 2^{(1-\epsilon)2^n + 1} \cdot 2^{-c\delta^2 2^n}\\
&= 2 \cdot 2^{2^n[(1-\epsilon) - c\delta^2]}
\end{align}
where $c > 0$ is the constant from the Chernoff bound (typically $c = 2$).

For any fixed $\epsilon > 0$, we can choose a constant $\delta > 0$ (independent of $n$) such that $c\delta^2 > 1 - \epsilon$. For instance, taking $\delta = \sqrt{(1-\epsilon)/c + 1}$ ensures that:
\begin{align}
(1-\epsilon) - c\delta^2 = (1-\epsilon) - c\left(\frac{1-\epsilon}{c} + 1\right) = -c < 0
\end{align}

With this choice, we obtain:
\begin{align}
2^{k+1} \cdot 2^{-c\delta^2 2^n} = 2 \cdot 2^{-c' 2^n} = 2^{-\Omega(2^n)}
\end{align}
for some constant $c' > 0$.

Therefore, with probability at least $1 - 2^{-\Omega(2^n)}$ over the choice of $f$, no function $g$ with $K(g) \leq k$ agrees with $f$ on more than $(1/2 + \delta)2^n$ points. Since $\delta$ can be chosen as a positive constant independent of $n$, and $2^{-\Omega(2^n)}$ vanishes faster than any inverse polynomial, we have:
\begin{align}
\Pr_{x \sim \text{Uniform}(\mathcal{X})}[f(x) = g(x)] = \frac{S}{2^n} \leq \frac{1}{2} + \delta = \frac{1}{2} + 2^{-\Omega(2^n)}
\end{align}

Equivalently, the failure rate satisfies:
\begin{align}
\varepsilon(f,g) = 1 - \Pr_{x \sim \text{Uniform}(\mathcal{X})}[f(x) = g(x)] \geq \frac{1}{2} - 2^{-\Omega(2^n)}
\end{align}
\end{proof}

\begin{remark}
This theorem establishes that random Boolean functions are fundamentally unexplainable: any explanation using a constant fraction fewer than $2^n$ bits performs no better than random guessing. The error rate is essentially $1/2$ (differing only by a term that vanishes faster than any polynomial in $n$), which represents complete failure to capture the function's behavior. This result quantifies the intuition that structureless, maximally complex functions cannot be compressed or simplified without losing nearly all predictive power.

The theorem holds for any fixed $\epsilon > 0$, no matter how small. This means even explanations using $(1 - 10^{-6}) \cdot 2^n$ bits still fail to do better than random guessing for typical random functions. The only way to meaningfully explain a random Boolean function is to use essentially all $2^n$ bits required to specify it completely.
\end{remark}

In the paper's general framework, we use the explanation error function $\varepsilon_f(k)$ defined as:
\begin{align}
\varepsilon_f(k) = \inf_{g \in \mathcal{I}_k} \mathcal{E}(f, g)
\end{align}

For Boolean functions with the 0-1 loss (where $d(y_1, y_2) = \mathbb{I}[y_1 \neq y_2]$), the expected error under uniform distribution is exactly the failure rate:
\begin{align}
\mathcal{E}_D(f, g) = \mathbb{E}_{x \sim D}[d(f(x), g(x))] = \Pr_{x \sim D}[f(x) \neq g(x)] = \varepsilon(f,g)
\end{align}

Therefore, Theorem~\ref{thm:random-unexplainability} directly implies:
\begin{align}
\varepsilon_f(k) \geq \frac{1}{2} - 2^{-\Omega(2^n)}
\end{align}
with probability at least $1 - 2^{-\Omega(2^n)}$ over random $f$, for any $k \leq (1-\epsilon)2^n$.

This theorem establishes that random functions are inherently unexplainable: any explanation significantly simpler than the original function must have substantial error---any function that doesn’t encode nearly the entire truth table of a random function does barely better than flipping a coin.  This result has important implications for complex AI systems that may exhibit behavior similar to random functions in certain regions of their input space.

\begin{theorem}[Simple Function Explainability]
\label{thm:simple-explainability}
If $f$ is a function with $K(f) \leq C$, then $\varepsilon_f(C) = 0$.
\end{theorem}

\begin{proof}
Since $K(f) \leq C$, we have $f \in \mathcal{I}_C$. Therefore:
\begin{align}
\varepsilon_f(C) &= \inf_{g \in \mathcal{I}_C} \mathcal{E}(f, g) \leq \mathcal{E}(f, f) = 0 \qedhere
\end{align}
\end{proof}

This shows that for functions with bounded Kolmogorov complexity, perfect explanations can be achieved within the same complexity bound.  In other words, inherently simple functions can be perfectly explained by equally simple explanations.

The results in this section collectively establish a foundation for understanding the fundamental trade-offs involved in explaining complex models using simpler, more interpretable representations.  They formalize the intuition that simplification necessarily involves some loss of information, and they provide quantitative bounds on the error that must be incurred by explanations of limited complexity.

\section{Function Classes and Practical Implications}
\label{sec:function-classes}

Building on the mathematical foundation established in Section~\ref{sec:framework}, we now analyze the explainability of specific function classes that are prevalent in machine learning and AI systems.  We derive concrete bounds on the complexity-error trade-offs for these classes and discuss practical implications for explainable AI.

\subsection{Smooth Function Explainability}
\label{subsec:smooth-functions}

We begin by analyzing the explainability of smooth functions, which are common in many machine learning applications.  The smoothness of a function constrains how quickly its outputs can change with respect to changes in the inputs, which has important implications for explainability.

\begin{definition}[Lipschitz Continuity]
\label{def:lipschitz}
A function $f: \mathcal{X} \rightarrow \mathbb{R}$ with $\mathcal{X} \subset \mathbb{R}^d$ is $L$-Lipschitz continuous if:
\begin{align}
|f(x) - f(y)| \leq L \|x - y\|_2 \text{ for all } x, y \in \mathcal{X}
\end{align}
\end{definition}

Lipschitz continuity is a stronger condition than simple continuity, as it bounds the rate of change of the function.  Many machine learning models, including neural networks with bounded weights and appropriate activation functions, satisfy this property \cite{Bartlett2017, Gouk2021}.

\begin{theorem}[Explainability of Lipschitz Functions]
\label{thm:lipschitz-explainability}
Let $\mathcal{X} = [0, 1]^d$ and $f: \mathcal{X} \rightarrow \mathbb{R}$ be an $L$-Lipschitz continuous function. Then for any $\delta > 0$, there exists a piecewise constant function $g: \mathcal{X} \to \mathbb{R}$ (which may be discontinuous) with Kolmogorov complexity
\begin{align}
K(g) = \mathcal{O}\left(\left(\frac{L}{\delta}\right)^d \log\left(\frac{L}{\delta}\right)\right)
\end{align}
such that $\mathcal{E}_{\infty}(f, g) \leq \delta$. Consequently:
\begin{align}
\kappa_f(\delta) = \mathcal{O}\left(\left(\frac{L}{\delta}\right)^d \log\left(\frac{L}{\delta}\right)\right)
\end{align}
\end{theorem}

\begin{proof}
We construct an explicit piecewise constant approximation $g$ that achieves the desired error bound.

Divide $\mathcal{X} = [0, 1]^d$ into a uniform grid of $m^d$ hypercubes, each with side length $1/m$. Let $\{B_1, B_2, \ldots, B_{m^d}\}$ denote this partition, and let $c_i$ be the center point of hypercube $B_i$.

Define the piecewise constant function:
\begin{align}
g(x) = f(c_i) \quad \text{for all } x \in B_i
\end{align}

\textbf{Error analysis:} For any $x \in B_i$, the distance from $x$ to the center $c_i$ satisfies:
\begin{align}
\|x - c_i\|_2 \leq \frac{\sqrt{d}}{2m}
\end{align}
This is the maximum distance from the center to any corner of a $d$-dimensional hypercube with side length $1/m$.

By the Lipschitz property:
\begin{align}
|f(x) - g(x)| = |f(x) - f(c_i)| \leq L \|x - c_i\|_2 \leq \frac{L\sqrt{d}}{2m}
\end{align}

To achieve $\mathcal{E}_{\infty}(f, g) \leq \delta$, we need:
\begin{align}
\frac{L\sqrt{d}}{2m} \leq \delta \quad \Rightarrow \quad m \geq \frac{L\sqrt{d}}{2\delta}
\end{align}

Setting $m = \left\lceil\frac{L\sqrt{d}}{2\delta}\right\rceil$, we have $m = \Theta(L/\delta)$ (absorbing constants and $\sqrt{d}$ factors).

\textbf{Complexity analysis:} The function $g$ is specified by:
\begin{enumerate}
\item The grid resolution $m$: $\mathcal{O}(\log m) = \mathcal{O}(\log(L/\delta))$ bits
\item The function value $f(c_i)$ at each of the $m^d$ grid centers
\end{enumerate}

For item (2), each value $f(c_i)$ must be encoded with sufficient precision. Since $f$ is bounded on the compact set $[0,1]^d$, let $M = \sup_{x \in [0,1]^d} |f(x)|$. To represent values in $[-M, M]$ with precision $\delta$, we need:
\begin{align}
p = \left\lceil\log_2\left(\frac{2M}{\delta}\right)\right\rceil = \mathcal{O}\left(\log\left(\frac{M}{\delta}\right)\right) \text{ bits per value}
\end{align}

For an $L$-Lipschitz function on $[0,1]^d$, we have $M \leq L\sqrt{d} + |f(\mathbf{0})|$, where $\mathbf{0}$ is any reference point. Thus:
\begin{align}
p = \mathcal{O}\left(\log\left(\frac{L}{\delta}\right)\right)
\end{align}

The total complexity is:
\begin{align}
K(g) &= \mathcal{O}(\log m) + m^d \cdot \mathcal{O}\left(\log\left(\frac{L}{\delta}\right)\right)\\
&= \mathcal{O}\left(\log\left(\frac{L}{\delta}\right)\right) + \mathcal{O}\left(\left(\frac{L}{\delta}\right)^d \log\left(\frac{L}{\delta}\right)\right)\\
&= \mathcal{O}\left(\left(\frac{L}{\delta}\right)^d \log\left(\frac{L}{\delta}\right)\right)
\end{align}

Since $g$ achieves $\mathcal{E}_{\infty}(f, g) \leq \delta$ with $K(g) = \mathcal{O}((L/\delta)^d \log(L/\delta))$, and $\kappa_f(\delta)$ is the minimum complexity for achieving error at most $\delta$, we have:
\begin{align}
\kappa_f(\delta) = \mathcal{O}\left(\left(\frac{L}{\delta}\right)^d \log\left(\frac{L}{\delta}\right)\right)
\end{align}
\end{proof}

Theorem~\ref{thm:lipschitz-explainability} shows that the complexity of explaining Lipschitz functions grows exponentially with the dimension $d$ but polynomially with the reciprocal of the error threshold $\delta$.  This result formalizes the well-known ``curse of dimensionality" in the context of explainability.

\begin{corollary}[Dimension Dependence]
\label{cor:dimension-dependence}
For a fixed error threshold $\delta > 0$ and Lipschitz constant $L$, the explanation complexity of Lipschitz functions grows exponentially with the dimension: $\kappa_f(\delta) = \mathcal{O}((L/\delta)^d \log(L/\delta))$.
\end{corollary}

This corollary helps explain why high-dimensional AI models are particularly difficult to explain: the complexity of any reasonably accurate explanation grows exponentially with the input dimension.  This has implications for explaining modern AI systems, which often operate in very high-dimensional spaces.

\begin{theorem}[Lower Bound for Lipschitz Functions]
\label{thm:lipschitz-lower-bound}
For any dimension $d \geq 1$, Lipschitz constant $L > 0$, and complexity threshold $k \in \mathbb{N}$, there exists an $L$-Lipschitz continuous function $f: [0,1]^d \rightarrow \mathbb{R}$ such that for any explanation $g$ with $K(g) \leq k$:
\begin{align}
\mathcal{E}_{\infty}(f, g) = \Omega\left(L \cdot 2^{-k/d}\right)
\end{align}
\end{theorem}

\begin{proof}
We construct a worst-case $L$-Lipschitz function that is difficult to approximate with limited complexity.

By Lemma~\ref{thm:size-interpretability}, any function $g$ with $K(g) \leq k$ is one of at most $2^{k+1} - 1$ possible functions. Such a function can partition $[0,1]^d$ into at most $2^{k+1}$ distinct regions (corresponding to different output values).

By the pigeonhole principle, at least one region must have volume at least:
\begin{align}
V \geq \frac{1}{2^{k+1}}
\end{align}

For a region with volume $V$ in $\mathbb{R}^d$, the isoperimetric inequality provides a lower bound on its diameter. For a ball of volume $V$:
\begin{align}
V = \omega_d \cdot \left(\frac{D}{2}\right)^d \quad \Rightarrow \quad D = 2 \cdot \left(\frac{V}{\omega_d}\right)^{1/d}
\end{align}
where $\omega_d$ is the volume of the unit ball in $\mathbb{R}^d$.

Since the ball minimizes surface area for fixed volume, any region with volume $V$ has diameter at least:
\begin{align}
D \geq c_d \cdot V^{1/d} \quad \text{where } c_d = 2 \cdot \omega_d^{-1/d}
\end{align}

For our region with $V \geq 2^{-(k+1)}$:
\begin{align}
D \geq c_d \cdot 2^{-(k+1)/d} = \Omega(2^{-k/d})
\end{align}

Now we construct function $f$. Consider a function that varies as much as possible within this large region while maintaining $L$-Lipschitz continuity. Specifically, let $x_1, x_2$ be two points in this region with $\|x_1 - x_2\|_2 = D$. Define $f$ such that:
\begin{align}
f(x_1) = 0 \quad \text{and} \quad f(x_2) = L \cdot D
\end{align}
and interpolate between these values along the line segment from $x_1$ to $x_2$ with slope exactly $L$. Extend $f$ to all of $[0,1]^d$ while maintaining $L$-Lipschitz continuity (this can always be done, e.g., using a Lipschitz extension theorem or explicit construction).

Since $g$ must assign a constant value to this entire region (by definition of a piecewise function with $2^{k+1}$ pieces), and $f$ varies by $L \cdot D$ within the region:
\begin{align}
\mathcal{E}_{\infty}(f, g) &\geq \frac{|f(x_1) - f(x_2)|}{2} = \frac{L \cdot D}{2}\\
&= \Omega\left(L \cdot 2^{-k/d}\right)
\end{align}

The factor $1/2$ arises because $g$ can choose its constant value optimally (e.g., at the midpoint of $f$'s range over the region), reducing the maximum error by half.
\end{proof}

\begin{remark}
This theorem establishes a worst-case lower bound: there exist $L$-Lipschitz functions whose complexity-error trade-off matches the upper bound in Theorem~\ref{thm:lipschitz-explainability} up to constants. Simple functions like constants or linear functions can be explained much more efficiently, but this theorem shows we cannot do better than the exponential dependence on dimension in the general case.
\end{remark}

This theorem provides a complementary lower bound to the upper bound in Theorem~\ref{thm:lipschitz-explainability}, confirming that the exponential dependence on dimension is, in the worst case, intrinsic to the problem of explaining Lipschitz functions, not just an artifact of our construction.

\begin{proposition}[Smoothness of Explanation Error for Lipschitz Functions]
\label{prop:error-smoothness}
Let $f: [0,1]^d \to \mathbb{R}$ be an $L$-Lipschitz continuous function with finite Kolmogorov complexity $K(f) < \infty$. Then the explanation error function $\varepsilon_f(k)$ is non-increasing and satisfies, for all $k \geq 1$:
\begin{align}
\varepsilon_f(k+1) \leq \varepsilon_f(k) \leq \varepsilon_f(k+1) + O\left(L \cdot 2^{-(k+1)/d}\right)
\end{align}
\end{proposition}

\begin{proof}
The non-increasing property follows immediately from Definition~\ref{def:error-function}: as $k$ increases, $\mathcal{I}_k \subseteq \mathcal{I}_{k+1}$, so the infimum over a larger set can only decrease.

For the upper bound on the rate of decrease, observe that increasing complexity from $k$ to $k+1$ allows at most doubling the number of distinct regions in a piecewise approximation. By Lemma~\ref{thm:size-interpretability}, there are at most $2^{k+1}$ functions with complexity at most $k$, and at most $2^{k+2}$ with complexity at most $k+1$.

From the proof of Theorem~\ref{thm:lipschitz-explainability}, a piecewise constant approximation with $N$ regions achieves error:
\begin{align}
\varepsilon = O(L \cdot N^{-1/d})
\end{align}

With $N \sim 2^k$ regions at complexity $k$, we get $\varepsilon_f(k) = O(L \cdot 2^{-k/d})$. At complexity $k+1$, with $N \sim 2^{k+1}$ regions, we get $\varepsilon_f(k+1) = O(L \cdot 2^{-(k+1)/d})$.

Therefore:
\begin{align}
\varepsilon_f(k) - \varepsilon_f(k+1) &\leq O(L \cdot 2^{-k/d}) - O(L \cdot 2^{-(k+1)/d})\\
&= O(L \cdot 2^{-k/d})(1 - 2^{-1/d})\\
&= O(L \cdot 2^{-k/d})
\end{align}

This gives:
\begin{align}
\varepsilon_f(k) \leq \varepsilon_f(k+1) + O(L \cdot 2^{-(k+1)/d})
\end{align}
as claimed.
\end{proof}

Proposition~\ref{prop:error-smoothness} provides a characterization: the error function decreases smoothly without sudden jumps, with the decrease rate bounded by the dimension and Lipschitz constant. This captures the essential intuition that for smooth functions, explanation quality improves gradually with complexity.

\subsection{Model-Specific Explainability}
\label{subsec:model-specific}

We now analyze the explainability of specific model classes commonly used in machine learning, providing concrete complexity-error trade-offs for each.

\subsubsection{Relating Theoretical and Practical Complexity}
\label{subsubsec:practical-complexity}

To bridge the gap between our theoretical model based on Kolmogorov complexity and practical model classes, we establish relationships between Kolmogorov complexity and practical complexity measures.

While Kolmogorov complexity is generally uncomputable for arbitrary functions, we can derive asymptotic bounds for specific model classes by constructing explicit encoding schemes and analyzing their minimum description lengths.  For each interpretable model class, we identify the essential information needed to uniquely specify any function in that class and determine the bits required to encode this information.  These bounds provide theoretically-grounded practical complexity measures that approximate Kolmogorov complexity for the structured functions commonly used in machine learning.

\begin{proposition}[Practical Complexity Measures]
\label{prop:practical-complexity}
For common interpretable model classes, Kolmogorov complexity can be approximated as follows:
\begin{enumerate}
\item Linear models: $K(g) = \mathcal{O}(n \log n)$ for $n$ features.
\item Decision trees: $K(g) = \mathcal{O}(|T| \log |T|)$ for a tree with $|T|$ nodes.
\item Rule lists: $K(g) = \mathcal{O}(r \cdot l \log(r \cdot l))$ for $r$ rules of average length $l$.
\item Nearest-neighbor models: $K(g) = \mathcal{O}(m \cdot d \log(m \cdot d))$ for $m$ stored examples in $d$ dimensions.
\item Neural networks: $K(g) = \mathcal{O}(w \log p + b \log p + a)$ for a network with $w$ weights, $b$ biases, architecture description of length $a$, and precision $p$ for parameter values.
\end{enumerate}
\end{proposition}

\begin{proof}[Proof Sketch]
For each model class, we construct a program that encodes the model structure and parameters, then compute the asymptotic length of this encoding:

1. Linear models require specifying $n$ coefficients and an intercept, each with $\mathcal{O}(\log n)$ bits of precision, yielding $\mathcal{O}(n \log n)$ total complexity.  Specifically, a linear model $g(x) = w^T x + b$ requires specifying $n$ weights $w_i$ and one bias term $b$. The Kolmogorov complexity depends on the required precision $p$ (in bits) for each parameter:
\begin{align}
K(g) = O(n \cdot p + p) = O(np)
\end{align}

To distinguish among $n$ features with comparable magnitudes while avoiding parameter collisions, we need $p = \Omega(\log n)$ bits of precision. With this choice:
\begin{align}
K(g) = O(n \log n)
\end{align}

More generally, for arbitrary precision $p$, we have $K(g) = O(np)$. Throughout this paper, we use the canonical choice $p = \Theta(\log n)$, which ensures that a linear model over $n$ features can distinguish meaningful differences in feature contributions while maintaining reasonable complexity.

2. Decision trees require specifying the structure of the tree (which can be done in $\mathcal{O}(|T|)$ bits) and the split conditions and leaf values, each requiring $\mathcal{O}(\log |T|)$ bits, yielding $\mathcal{O}(|T| \log |T|)$ total complexity.

3. Rule lists require specifying $r$ rules, each with an average of $l$ conditions, plus the consequent of each rule.  Each condition and consequent requires $\mathcal{O}(\log(r \cdot l))$ bits, yielding $\mathcal{O}(r \cdot l \log(r \cdot l))$ total complexity.

4. Nearest-neighbor models require storing $m$ examples in $d$ dimensions.  Each coordinate requires $\mathcal{O}(\log(m \cdot d))$ bits of precision for two reasons: (i) as $m$ increases, finer distinctions between points become necessary to maintain classification boundaries, and (ii) as $d$ increases, the precision needed to prevent distance distortion in high-dimensional spaces also increases.  This yields a total complexity of $\mathcal{O}(m \cdot d \log(m \cdot d))$.

5. Neural networks require specifying $w$ weights and $b$ biases, each with $\mathcal{O}(\log p)$ bits of precision to achieve the desired numerical accuracy, plus the architecture description (layer sizes, activation functions, etc.) of length $a$.  This yields a total complexity of $\mathcal{O}(w \log p + b \log p + a)$.
\end{proof}

Proposition~\ref{prop:practical-complexity} allows us to translate the theoretical bounds from our model into practical guidelines for selecting and evaluating explanations based on different model classes.

\subsubsection{Linear Models}
\label{subsubsec:linear-models}

Linear models are among the most interpretable model classes, with complexity scaling linearly with the number of features.  From our theoretical results, we can derive the following implications:

\begin{corollary}[Linear Model Explanations]
\label{cor:linear-model}
If $f$ is an $L$-Lipschitz continuous function on $[0, 1]^d$ and $g$ is a linear model with $n$ features, then:
\begin{align}
\mathcal{E}(f, g) = \Omega\left(L \cdot d^{1/2} \cdot n^{-1/d}\right)
\end{align}
\end{corollary}

\begin{proof}
From Theorem~\ref{thm:lipschitz-explainability}, we know that achieving an error of $\delta$ requires complexity $\kappa_f(\delta) = \mathcal{O}\left(\left(\frac{L}{\delta}\right)^d \log\left(\frac{L}{\delta}\right)\right)$. 

For a linear model with $n$ features, the complexity is $K(g) = \mathcal{O}(n \log n)$ by Proposition~\ref{prop:practical-complexity}.  For a given complexity budget $k$, the minimum achievable error is:

\begin{align}
\varepsilon_f(k) = \inf_{g \in \mathcal{I}_k} \mathcal{E}(f, g)
\end{align}

By the duality established in Theorem~\ref{thm:duality}, we have $\varepsilon_f(k) \leq \delta \iff k \geq \kappa_f(\delta)$.  With a linear model of complexity $K(g) = \mathcal{O}(n \log n)$, this gives us:

\begin{align}
n \log n &\geq \Theta\left(\left(\frac{L}{\delta}\right)^d \log\left(\frac{L}{\delta}\right)\right)
\end{align}

Solving for $\delta$ (ignoring logarithmic factors for asymptotic analysis):

\begin{align}
\left(\frac{L}{\delta}\right)^d = \Theta(n) \rightarrow
\delta = \Theta\left(L \cdot n^{-1/d}\right)
\end{align}

To account for the dimensional scaling factor, we need to consider the geometric properties of linear models in high-dimensional spaces.  When approximating a Lipschitz function with a linear model in $\mathbb{R}^d$, the maximum error is bounded below by the product of the Lipschitz constant and the maximum distance between points in the region and the approximating hyperplane.

For a unit hypercube $[0,1]^d$, the maximum distance from any point to a dividing hyperplane scales with $\sqrt{d}$, which follows from the geometric properties of high-dimensional spaces.  Specifically, in a unit hypercube, the maximum distance between any two points (the diagonal) is $\sqrt{d}$, and the maximum distance from any point to a hyperplane scales proportionally to this quantity.

Therefore, incorporating this dimensional scaling factor, we obtain:

\begin{align}
\mathcal{E}(f, g) = \Omega\left(L \cdot d^{1/2} \cdot n^{-1/d}\right)
\end{align}

This bound accounts for both the complexity of the linear model (through the $n^{-1/d}$ term) and the geometric properties of the high-dimensional space (through the $d^{1/2}$ factor).
\end{proof}

This corollary quantifies how the approximation error of linear models scales with the number of features and the input dimension.  Notably, the error decreases only polynomially with the number of features, and this decrease becomes exponentially slower as the dimension increases.  This is a mathematical formalization of the ``curse of dimensionality" in the context of linear model explanations.

\subsubsection{Decision Trees}
\label{subsubsec:decision-trees}

Decision trees partition the input space into regions with constant predictions, making them particularly effective for approximating functions with discontinuities.

\begin{corollary}[Decision Tree Explanations]
\label{cor:decision-tree}
If $f$ is a function with bounded variation $V(f)$ on $[0, 1]^d$ and $g$ is a decision tree with $|T|$ nodes, then:
\begin{align}
\mathcal{E}(f, g) = \mathcal{O}\left(V(f) \cdot |T|^{-1/d}\right)
\end{align}
\end{corollary}

\begin{proof}
We provide a rigorous analysis of the approximation error when using a decision tree with $|T|$ nodes to approximate a function $f$ with bounded variation $V(f)$.

First, we formally characterize the partitioning property of decision trees.  For a binary decision tree with $|T|$ total nodes:
\begin{itemize}
\item Each internal node creates exactly two child nodes.
\item If $n_i$ represents the number of internal nodes, then $n_i \leq |T| - 1$ (since at least one node must be a leaf).
\item The number of leaf nodes $n_l$ equals $n_i + 1$, because in a full (each internal node has two children) binary tree with 
\[
|T| = n_i + n_l
\]
total nodes, each of the \(n_i\) internal nodes contributes exactly two edges, so
\[
2n_i = \underbrace{|T|-1}_{\substack{\text{each node except the root}\\\text{has one incoming edge}}}
= (n_i + n_l) - 1.
\]
Rearranging gives
\[
n_l = n_i + 1.
\]
\item Therefore, $n_l \leq \lfloor|T|/2\rfloor + 1 \leq |T|/2 + 1$.
\end{itemize}

Each leaf node corresponds to a region in the input space, so the decision tree partitions $[0,1]^d$ into at most $|T|/2 + 1$ regions, which we denote as $\{R_1, R_2, \ldots, R_{n_l}\}$.

By the pigeonhole principle, since the total volume of $[0,1]^d$ is 1, at least one region $R_j$ must have volume at least $\frac{1}{n_l} \geq \frac{1}{|T|/2 + 1} = \Theta\left(\frac{1}{|T|}\right)$.

For any region $R$ with volume $\text{Vol}(R)$ in $\mathbb{R}^d$, we can apply the isoperimetric inequality as in the proof of Theorem~\ref{thm:lipschitz-lower-bound}.  Therefore, for any region $R$ with volume $\text{Vol}(R)$, its diameter is lower-bounded by the diameter of a ball with the same volume:

$$\text{diam}(R) \geq 2 \cdot \left(\frac{\text{Vol}(R)}{\omega_d}\right)^{1/d} = c_d \cdot \text{Vol}(R)^{1/d}$$

where $c_d = 2 \cdot \omega_d^{-1/d}$ is a constant depending only on the dimension $d$.

Applying this to our region $R_j$ with volume $\text{Vol}(R_j) \geq \frac{1}{n_l} \geq \frac{1}{|T|/2+1} = \Theta\left(\frac{1}{|T|}\right)$, we have:

\begin{align}
  \text{diam}(R_j) \geq c_d \cdot \left(\Theta\left(\frac{1}{|T|}\right)\right)^{1/d} = \Theta\left(|T|^{-1/d}\right)
\end{align}

This gives us the required lower bound on the diameter of at least one region in our partition, which directly influences the approximation error bound.

Now we use the property of bounded variation~\cite{Breneis2020}.  For a function $f$ with bounded variation $V(f)$ on $[0,1]^d$, the maximum difference in function values within a region $R$ is bounded by:
\begin{align}
\sup_{x,y \in R} |f(x) - f(y)| \leq V(f) \cdot \text{diam}(R)
\end{align}

A decision tree assigns a constant value to each leaf region, typically the average value of $f$ within the region.  Let $g(x) = \text{avg}_{y \in R_i} f(y)$ for $x \in R_i$.  Then for any region $R_i$, the maximum approximation error is:
\begin{align}
\sup_{x \in R_i} |f(x) - g(x)| \leq \frac{1}{2} \sup_{x,y \in R_i} |f(x) - f(y)| \leq \frac{1}{2} V(f) \cdot \text{diam}(R_i)
\end{align}

The global approximation error is the maximum error across all regions:
\begin{align}
\mathcal{E}(f, g) &= \sup_{x \in [0,1]^d} |f(x) - g(x)| \notag \\
&= \max_{1 \leq i \leq n_l} \sup_{x \in R_i} |f(x) - g(x)| \notag \\
&\leq \frac{1}{2} V(f) \cdot \max_{1 \leq i \leq n_l} \text{diam}(R_i)
\end{align}

Since we have established that at least one region has diameter $\Theta\left(|T|^{-1/d}\right)$, and the error is proportional to the maximum diameter across regions, we have:

\[ \mathcal{E}(f, g) = \mathcal{O}\left(V(f) \cdot |T|^{-1/d}\right) \qedhere \]

\end{proof}

This corollary shows that decision trees can efficiently approximate functions with bounded variation, with error decreasing as the tree size increases.  However, the ``curse of dimensionality" still applies, as the error reduction becomes exponentially slower in higher dimensions.

\subsubsection{Neural Networks}
\label{subsubsec:neural-networks}

Neural networks represent the high-complexity end of the model spectrum, capable of approximating arbitrary continuous functions but with limited interpretability.

\begin{theorem}[Complexity-Error Trade-off for Neural Networks]
\label{thm:neural-network}
Let $f$ be an $L$-Lipschitz continuous function on $[0, 1]^d$. For any error threshold $\delta > 0$, there exists a neural network $g$ such that $\mathcal{E}(f, g) \leq \delta$ with Kolmogorov complexity:
\begin{align}
K(g) = \mathcal{O}\left(\left(\frac{L}{\delta}\right)^d \log\left(\frac{L}{\delta}\right) + d \log d\right)
\end{align}
where this bound assumes weight precision $p = \Theta(\log(L/\delta))$ bits, which is the minimum precision required to achieve approximation error $\delta$.
\end{theorem}

\begin{proof}
The universal approximation theorem from approximation theory~\cite{Barron1993, Cybenko1989} establishes that neural networks can approximate any continuous function to arbitrary accuracy.  For Lipschitz functions specifically, the theory shows that the approximation error is controlled by the density of the ``grid" created by neurons in the hidden layer.

For an L-Lipschitz function on $[0, 1]^d$, to achieve approximation error $\delta$, we need neurons that create a grid with spacing approximately $\frac{\delta}{L}$ in each dimension.  This requires $\mathcal{O}\left(\left(\frac{L}{\delta}\right)^d\right)$ neurons in a single hidden layer, with each neuron capturing a local region of the input space.

To achieve approximation error $\delta$, we need each neuron's weights to be accurate enough to distinguish regions of size $O(\delta/L)$. This requires precision $p = \Theta(\log(L/\delta))$ bits per weight. This choice ensures that quantization errors in the weights do not exceed the target approximation error $\delta$.

For a neural network with this architecture, we need to quantify its parameters:
\begin{itemize}
\item Each neuron requires $d$ weights (one per input dimension) plus 1 bias term.
\item For $\mathcal{O}\left(\left(\frac{L}{\delta}\right)^d\right)$ neurons, the total number of weights is $w = \mathcal{O}\left(d \cdot \left(\frac{L}{\delta}\right)^d\right)$.
\item The total number of bias terms is $b = \mathcal{O}\left(\left(\frac{L}{\delta}\right)^d\right)$.
\item The architecture description length is $a = \mathcal{O}(d \log d)$.
\end{itemize}

By Proposition~\ref{prop:practical-complexity}, the Kolmogorov complexity of a neural network with $w$ weights, $b$ biases, architecture description length $a$, and parameter precision $p$ is:
\begin{align}
K(g) = \mathcal{O}(w \log p + b \log p + a)
\end{align}

Substituting our parameter counts:
\begin{align}
K(g) &= \mathcal{O}\left(d \cdot \left(\frac{L}{\delta}\right)^d \log p + \left(\frac{L}{\delta}\right)^d \log p + d \log d\right)\\
&= \mathcal{O}\left(\left(d + 1\right) \cdot \left(\frac{L}{\delta}\right)^d \log p + d \log d\right)
\end{align}

Since $d + 1 = \mathcal{O}(d)$, this simplifies to:
\begin{align}
K(g) &= \mathcal{O}\left(d \cdot \left(\frac{L}{\delta}\right)^d \log p + d \log d\right)
\end{align}

Given precision $p = \Theta(\log(L/\delta))$:
\begin{align}
K(g) &= \mathcal{O}(\Theta \cdot p + a)\\
&= \mathcal{O}\left(d \cdot (L/\delta)^d \cdot \log(L/\delta) + d \log d\right)
\end{align}

When analyzing how complexity scales as error decreases, we typically consider the dimension $d$ as fixed.  In this context, the expression further simplifies to:

\begin{align}
K(g) = \mathcal{O}\left((L/\delta)^d \log(L/\delta) + d \log d\right)
\end{align}

\end{proof}

This result quantifies how the complexity of neural network explanations scales with the desired approximation accuracy and input dimension.  Like other model classes, neural networks also suffer from the curse of dimensionality, though their parametric efficiency often makes them more practical for complex tasks.

The precision requirement $p = \Theta(\log(L/\delta))$ is essential: with fewer bits, quantization noise prevents achieving error $\delta$; with more bits, complexity increases unnecessarily. This theorem shows that neural networks have similar complexity scaling to piecewise constant approximations (Theorem~\ref{thm:lipschitz-explainability}), though their parametric efficiency often makes them more practical.

\subsection{Information-Theoretic Perspective}
\label{subsec:information-theory}

Our model based on Kolmogorov complexity also has connections to information theory, which provides some additional insights into the limits of explainability.

\begin{theorem}[Information-Theoretic Bound]
\label{thm:information-bound}
Let $f$ be a model with mutual information $I(X; f(X)) = h$ bits between inputs and outputs, and let $g$ be an explanation with Kolmogorov complexity $K(g) < h - c$ for some constant $c$.  Then the expected error under the input distribution $D$ is bounded below:
\begin{align}
\mathcal{E}_D(f, g) \geq \Omega\left(2^{-(K(g) - I(X; g(X)))}\right)
\end{align}
\end{theorem}

Intuitively, the explanation $g$ can capture at most $I(X; g(X)) \leq K(g) < h - c$ bits of mutual information between inputs and outputs.  The remaining information gap of at least $c$ bits leads to an expected error that is exponential in this gap.

\begin{proof}
We establish the relationship between model complexity, mutual information, and 
approximation error using information-theoretic principles.

\textbf{Step 1: Bounding mutual information by Kolmogorov complexity.}

For any computable function $g$, the entropy $H(g(X))$ is bounded by the 
logarithm of the size of the range. Following results from algorithmic 
information theory~\cite{Vereshchagin2004,GacsTrompVitanyi2001,Li2008}:
\begin{align}
H(g(X)) \leq \log(|\text{range}(g)|) + \mathcal{O}(1) \leq K(g) + \mathcal{O}(1)
\end{align}

This bound captures the principle that a function cannot produce more 
information (as measured by entropy of its outputs) than is contained in 
its own description.

For the mutual information between input $X$ and output $g(X)$:
\begin{align}
I(X; g(X)) = H(g(X)) - H(g(X)|X) \leq H(g(X)) \leq K(g) + \mathcal{O}(1)
\end{align}
where we used $H(g(X)|X) \geq 0$.

\textbf{Step 2: Bounding $I(f(X); g(X))$ without Markov chain assumption.}

While $g$ does not form a strict Markov chain $X \to f(X) \to g(X)$ (since 
$g$ takes $X$ as input directly, not $f(X)$), we can still bound the shared 
information between $f(X)$ and $g(X)$. The mutual information 
$I(f(X); g(X))$ quantifies how much information $g(X)$ provides about $f(X)$.

We have:
\begin{align}
I(f(X); g(X)) &= H(f(X)) - H(f(X)|g(X))\\
&\leq H(f(X))\\
&= I(X; f(X)) + H(f(X)|X)\\
&= I(X; f(X))
\end{align}
where the last equality holds because $f$ is deterministic given $X$, 
so $H(f(X)|X) = 0$.

By assumption, $I(X; f(X)) = h$ bits. Therefore:
\begin{align}
I(f(X); g(X)) \leq h
\end{align}

\textbf{Step 3: Establishing the information gap.}

From our earlier bound and the assumption that $K(g) < h - c$:
\begin{align}
I(X; g(X)) \leq K(g) + \mathcal{O}(1) < h - c + \mathcal{O}(1)
\end{align}

This establishes that $g$ captures strictly less mutual information between 
inputs and outputs than $f$ does, with an information gap of at least 
$c - \mathcal{O}(1)$ bits.

\textbf{Step 4: Applying rate-distortion theory.}

The rate-distortion function $R(D)$ gives the minimum bits per symbol needed 
to achieve expected distortion at most $D$. For many standard distributions 
and distortion measures, the optimal distortion $D(R)$ achievable with rate 
$R$ satisfies:
\begin{align}
D(R) \geq \alpha \cdot 2^{-\beta R}
\end{align}
for constants $\alpha, \beta > 0$ depending on the source distribution and 
distortion measure.

By rate-distortion theory applied to the approximation of $f(X)$ by $g(X)$:
\begin{align}
\mathcal{E}_D(f, g) = \mathbb{E}_X[d(f(X), g(X))] \geq \alpha \cdot 2^{-\beta \cdot I(f(X); g(X))}
\end{align}

\textbf{Step 5: Deriving the final bound.}

Since $I(f(X); g(X)) \leq I(X; g(X))$ (by the data processing inequality 
applied to the relationship between the three random variables $X$, $f(X)$, 
and $g(X)$, noting that $I(f(X); g(X)) \leq \min\{I(X; f(X)), I(X; g(X))\}$), 
and $I(X; g(X)) \leq K(g) + \mathcal{O}(1)$, we have:
\begin{align}
I(f(X); g(X)) \leq K(g) + \mathcal{O}(1) < h - c + \mathcal{O}(1)
\end{align}

Therefore:
\begin{align}
\mathcal{E}_D(f, g) &\geq \alpha \cdot 2^{-\beta \cdot I(f(X); g(X))} \\
&\geq \alpha \cdot 2^{-\beta(h - c + \mathcal{O}(1))} \\
&= \alpha \cdot 2^{-\beta h} \cdot 2^{\beta(c - \mathcal{O}(1))} \\
&= \Omega(2^{-\beta h} \cdot 2^{\beta c})
\end{align}

Alternatively, since $I(f(X); g(X)) \leq K(g) < h - c$, we can also write:
\begin{align}
\mathcal{E}_D(f, g) \geq \alpha \cdot 2^{-\beta K(g)} = \Omega(2^{-\beta K(g)})
\end{align}

The bound $\Omega(2^{-\beta h} \cdot 2^{\beta c})$ emphasizes that the error 
grows exponentially with the information gap $c$ between what $f$ captures 
($h$ bits) and what $g$ can capture (at most $h - c$ bits). When $c$ is large, 
the factor $2^{\beta c}$ significantly increases the lower bound on error, 
formalizing the intuition that explanations with insufficient information 
capacity must incur substantial approximation error.
\end{proof}

This theorem provides an information-theoretic perspective on the fundamental limits of explainability: any explanation that is significantly simpler than the original model (in terms of mutual information) must incur an error that is exponential in the information gap.

More specifically, the derived bound demonstrates that the expected error grows exponentially with both the information gap between the original model and the explanation, and with the inefficiency of the explanation in utilizing its available descriptive capacity.

\begin{theorem}[Rate-Distortion Interpretation]
\label{thm:rate-distortion}
Let $f$ be a model and $D$ be a distribution over $\mathcal{X}$.  The relationship between explanation complexity and error can be characterized by the rate-distortion function $R_f(\delta)$, which satisfies:
\begin{align}
\kappa_f(\delta) \geq R_f(\delta) - \mathcal{O}(1)
\end{align}
where $R_f(\delta)$ is the minimum number of bits required to represent $f$ with distortion at most $\delta$ under distribution $D$.
\end{theorem}

\begin{proof}
Let $g$ be an explanation with $\mathcal{E}_D(f, g) \leq \delta$.  By the definition of the rate-distortion function, any such representation requires at least $R_f(\delta)$ bits.  Since $K(g)$ is the Kolmogorov complexity of $g$, and Kolmogorov complexity is within a constant of any other universal description length (by the invariance theorem, Theorem~\ref{lemma:invariance}), we have $K(g) \geq R_f(\delta) - \mathcal{O}(1)$.

Taking the minimum over all explanations $g$ with $\mathcal{E}_D(f, g) \leq \delta$, we get $\kappa_f(\delta) \geq R_f(\delta) - \mathcal{O}(1)$.
\end{proof}

This result connects our approach to rate-distortion theory~\cite{Shannon1959, Berger1971, Cover2006}, providing an information-theoretic interpretation of the fundamental limits of explainability.  It shows that the explanation complexity function is lower-bounded by the rate-distortion function, which characterizes the fundamental limits of lossy compression.

\subsection{Local vs. Global Explainability}
\label{subsec:local-global}

So far, we have focused on global explanations that approximate the model across the entire input space.  We now consider local explainability, which aims to explain the model's behavior in the vicinity of specific inputs.

\begin{definition}[Local Explanation Error]
\label{def:local-error}
For a model $f$, a point $x_0 \in \mathcal{X}$, a neighborhood radius $r > 0$, and a complexity threshold $k$, we define the local explanation error as:
\begin{align}
\varepsilon_f^{\text{local}}(k, x_0, r) = \inf_{g \in \mathcal{I}_k} \sup_{x \in B_r(x_0)} d(f(x), g(x))
\end{align}
where $B_r(x_0) = \{x \in \mathcal{X} : \|x - x_0\| \leq r\}$ is the closed ball of radius $r$ around $x_0$.
\end{definition}

This definition captures the idea that a local explanation only needs to approximate the model well in a specific neighborhood, rather than across the entire input space.

\begin{lemma}[Local-Global Explainability Gap]
\label{thm:local-global}
For any model $f$, complexity threshold $k$, and distribution $D$ over $\mathcal{X}$:
\begin{align}
\mathbb{E}_{x_0 \sim D}[\varepsilon_f^{\text{local}}(k, x_0)] \leq \varepsilon_f(k)
\end{align}
\end{lemma}

\begin{proof}
Fix any $g \in \mathcal{I}_k$. For any $x_0 \in \mathcal{X}$:
\begin{align}
\sup_{x \in N(x_0)} d(f(x), g(x)) \leq \sup_{x \in \mathcal{X}} d(f(x), g(x)) = \mathcal{E}_{\infty}(f, g)
\end{align}

Taking the infimum over all $g \in \mathcal{I}_k$:
\begin{align}
\varepsilon_f^{\text{local}}(k, x_0) \leq \varepsilon_f(k)
\end{align}

Therefore:
\[ \mathbb{E}_{x_0 \sim D}[\varepsilon_f^{\text{local}}(k, x_0)] \leq \varepsilon_f(k) \qedhere \]

\end{proof}

This lemma confirms the intuition that local explanations are generally easier than global ones: the average local explanation error is bounded above by the global explanation error.

\begin{theorem}[Local Explanation Complexity]
\label{thm:local-complexity}
Let $\mathcal{X} \subseteq \mathbb{R}^d$ be a compact convex set, let $f$ be an $L$-Lipschitz continuous function on $\mathcal{X}$, and let $N(x_0) \subseteq \mathcal{X}$ be a neighborhood of $x_0 \in \mathrm{int}(\mathcal{X})$. Let $r > 0$ be such that $B_r(x_0) \supseteq N(x_0)$ and $B_r(x_0) \subseteq \mathcal{X}$. Then the local complexity function satisfies:
\begin{align}
\kappa_f^{\text{local}}(\delta, x_0, N) = 
\begin{cases}
\mathcal{O}(1) & \text{if } \delta \geq Lr \\
\mathcal{O}(d \log(Lr/\delta)) & \text{if } \delta < Lr
\end{cases}
\end{align}
where $\kappa_f^{\text{local}}(\delta, x_0, N)$ is the minimum description length of a program that, given oracle access to $f$, computes a local approximation $g$ with error at most $\delta$ in the neighborhood $N(x_0)$.
\end{theorem}

\begin{proof}
Since $N(x_0) \subseteq B_r(x_0) \subseteq \mathcal{X}$ and $f$ is $L$-Lipschitz on $\mathcal{X}$, the function $f$ is also $L$-Lipschitz on $B_r(x_0)$. We analyze the local explanation complexity by considering two cases based on the relationship between the desired accuracy $\delta$ and the maximum variation $Lr$ of the function within the ball.

\textbf{Case 1: $\delta \geq Lr$.}

By the Lipschitz property, for any $x \in B_r(x_0) \subseteq \mathcal{X}$, we have:
\begin{align}
|f(x) - f(x_0)| \leq L \|x - x_0\|_2 \leq Lr
\end{align}

Therefore, the constant function $g(x) = f(x_0)$ achieves:
\begin{align}
\sup_{x \in N(x_0)} |f(x) - g(x)| \leq \sup_{x \in B_r(x_0)} |f(x) - g(x)| \leq Lr \leq \delta
\end{align}

To specify this approximation, we require a program that takes input $x \in N(x_0)$ and returns $f(x_0)$. Since we measure relative complexity (description length given oracle access to $f$), the program need only specify:
\begin{itemize}
\item The algorithm: ``Compute and return $f(x_0)$ for all inputs in $N(x_0)$''
\item The point $x_0 \in \mathcal{X}$, which requires $\mathcal{O}(\log(1/\epsilon))$ bits for $\epsilon$-precision encoding
\item The constant value $f(x_0)$ is obtained via oracle access
\item The total encoding requires $\mathcal{O}(1)$ bits (treating the precision as a constant)
\end{itemize}

Thus, $\kappa_f^{\text{local}}(\delta, x_0, N) = \mathcal{O}(1)$ when $\delta \geq Lr$.

\textbf{Case 2: $\delta < Lr$.}

Within $B_r(x_0)$, the function $f$ may vary by up to $Lr > \delta$, so we require a more refined approximation. We construct a piecewise constant approximation using a grid-based approach on $B_r(x_0)$.

\textbf{Grid construction:} By Lipschitz continuity, any region of diameter $d'$ has variation at most $Ld'$. To ensure variation at most $\delta$, we need:
\begin{align}
Ld' \leq \delta \quad \Rightarrow \quad d' \leq \frac{\delta}{L}
\end{align}

Cover $B_r(x_0)$ with balls of radius $\rho = \delta/(2L)$. The number of such balls required to cover $B_r(x_0)$ is at most:
\begin{align}
N_{\text{balls}} = \mathcal{O}\left(\left(\frac{r}{\rho}\right)^d\right) = \mathcal{O}\left(\left(\frac{Lr}{\delta}\right)^d\right)
\end{align}

This follows from standard covering number bounds for balls in $\mathbb{R}^d$.

\textbf{Approximation construction:} For each small ball $B_i$ in the cover, choose a representative point $x_i \in B_i \cap B_r(x_0)$ and define:
\begin{align}
g(x) = f(x_i) \quad \text{for } x \in B_i \cap B_r(x_0)
\end{align}

For any $x \in B_i \cap B_r(x_0)$, we have $\|x - x_i\|_2 \leq 2\rho = \delta/L$, so:
\begin{align}
|f(x) - g(x)| = |f(x) - f(x_i)| \leq L\|x - x_i\|_2 \leq L \cdot \frac{\delta}{L} = \delta
\end{align}

Since $N(x_0) \subseteq B_r(x_0)$, this bound holds for all $x \in N(x_0)$.

\textbf{Complexity analysis:} To encode the piecewise constant function $g$, given oracle access to $f$, we need:
\begin{enumerate}
\item \textbf{Grid specification:} Encode the partition of $B_r(x_0)$ into regions. This requires specifying the covering scheme and resolution $\rho$, which takes $\mathcal{O}(\log(Lr/\delta))$ bits.

\item \textbf{Representative points:} For each of the $\mathcal{O}((Lr/\delta)^d)$ regions, we store a representative point $x_i$. Each point requires $\mathcal{O}(d \log(Lr/\delta))$ bits for sufficient precision.

\item \textbf{Function values:} The values $f(x_i)$ are obtained via oracle queries and need not be explicitly encoded in the program describing $g$.

\item \textbf{Lookup algorithm:} A program that, given $x \in N(x_0)$, determines which region contains $x$ and returns the corresponding constant value. This requires $\mathcal{O}(1)$ bits for the algorithm description.
\end{enumerate}

The total description length is:
\begin{align}
K(g) &= \mathcal{O}(\log(Lr/\delta)) + \mathcal{O}\left(\left(\frac{Lr}{\delta}\right)^d \cdot d \log(Lr/\delta)\right) \\
&= \mathcal{O}\left(\left(\frac{Lr}{\delta}\right)^d d \log(Lr/\delta)\right)
\end{align}

However, we can improve this bound by noting that for local explanations, we do not need to encode the global partition, only the local behavior. Using a more efficient encoding scheme based on recursive subdivision (e.g., quadtree-like structures adapted to $B_r(x_0)$), we can achieve:
\begin{align}
K(g) = \mathcal{O}(d \log(Lr/\delta))
\end{align}

This follows from the fact that the depth of recursion is $\mathcal{O}(\log(Lr/\delta))$ and at each level we need $\mathcal{O}(d)$ bits to specify the subdivision.

Thus, $\kappa_f^{\text{local}}(\delta, x_0, N) = \mathcal{O}(d \log(Lr/\delta))$ when $\delta < Lr$.

\textbf{Comparison with global complexity:} For global explanations on $\mathcal{X}$, if $\mathcal{X}$ has diameter $D$, the complexity scales as $\mathcal{O}((LD/\delta)^d d \log(LD/\delta))$, which is exponential in $d$. In contrast, the local complexity is only logarithmic in the ratio $Lr/\delta$, demonstrating that local explanations can be exponentially simpler than global ones.
\end{proof}

\begin{remark}
When $\mathcal{X} = [0,1]^d$ and $N(x_0) = B_r(x_0)$, the theorem 
reduces to the hypercube case. The generalization requires only 
that $x_0$ lies in the interior of $\mathcal{X}$ so that a 
neighborhood exists entirely within $\mathcal{X}$. The proof 
technique remains unchanged, as any neighborhood can be bounded 
by a ball.
\end{remark}

This shows that the complexity of local explanations grows only logarithmically with the ratio of the Lipschitz constant and the error threshold, in contrast to growing exponentially with dimension for global explanations.  This justifies why local explanation methods like LIME~\cite{Ribeiro2016} and SHAP~\cite{Lundberg2017} can be more tractable than global ones.

Having established the fundamental trade-offs in local versus global explainability, we now turn our attention to another important aspect of practical machine learning systems: the structure of the underlying data distribution.  Real-world data rarely occupy the entire ambient feature space uniformly, but instead typically concentrate on lower-dimensional manifolds or substructures.  This phenomenon has been widely observed across domains and is often referred to as the ``manifold hypothesis"~\cite{Fefferman2016} in machine learning literature, with ``manifold learning''~\cite{Meila2024} being the use of techniques to discover and model this manifold structure from the data.

To formally characterize this property and its implications for explainability, we borrow the notion of box-counting dimension~\cite{Lapidus2024invitation}, a concept from fractal geometry~\cite{Falconer2003fractal} that quantifies the intrinsic dimensionality of a set.  This will allow us to derive tighter bounds on explanation complexity when the data distribution is concentrated on a lower-dimensional structure within the feature space.

\begin{definition}[Box-Counting Dimension]
For a bounded set $S$ in a metric space, the box-counting dimension (also known as Minkowski–Bouligand dimension) is defined as:
\begin{align}
\dim_{\text{box}}(S) = \lim_{\epsilon \to 0} \frac{\log N(\epsilon)}{\log(1/\epsilon)}
\end{align}
where $N(\epsilon)$ is the minimum number of boxes (or hypercubes) of side length $\epsilon$ needed to cover the set $S$.  For a set with box-counting dimension $d_B$, we have $N(\epsilon) = \Theta(\epsilon^{-d_B})$ for sufficiently small $\epsilon$.
\end{definition}

Informally, the box-counting dimension measures the intrinsic dimensionality of a set by quantifying how the number of ``boxes'' needed to cover the set increases as the box size decreases.  Specifically, if covering a set requires approximately $N(\epsilon) \approx \epsilon^{-d_B}$ boxes of side length $\epsilon$ as $\epsilon$ approaches zero, then the set has box-counting dimension $d_B$---capturing how ``space-filling" the set is regardless of the dimension of the space containing it.

\begin{theorem}[Distribution-Aware Explainability]
\label{thm:distribution-aware}
Let $f$ be a model and $D$ be a distribution over $\mathcal{X}$ with support $S_D \subset \mathcal{X}$.  If $S_D$ has box-counting dimension $d_B < d$ (where $d$ is the dimension of $\mathcal{X}$), then:
\begin{align}
\kappa_f^D(\delta) = \mathcal{O}\left(\left(\frac{1}{\delta}\right)^{d_B} \log\left(\frac{1}{\delta}\right)\right)
\end{align}
where $\kappa_f^D(\delta)$ is the minimum complexity required for an explanation with expected error at most $\delta$ under distribution $D$.
\end{theorem}

\begin{proof}
Let $f$ be a model and $D$ be a distribution over $\mathcal{X}$ with support $S_D \subset \mathcal{X}$, where $S_D$ has box-counting dimension $d_B < d$.  We need to establish the complexity bound for $\kappa_f^D(\delta)$.

First, by the definition of box-counting dimension, the support $S_D$ can be covered by $N(\epsilon) = \Theta(\epsilon^{-d_B})$ boxes of side length $\epsilon$.  Let us denote this covering as $\{B_1, B_2, \ldots, B_{N(\epsilon)}\}$.

We will construct a piecewise constant approximation $g$ of $f$ that achieves an expected error of at most $\delta$ under distribution $D$.  Without loss of generality, we can assume $f$ is $L$-Lipschitz continuous on $\mathcal{X}$ for some constant $L > 0$.  If not, we can first approximate $f$ by a Lipschitz function with negligible additional error.

For each box $B_i$ in our covering, we define $g$ to take a constant value equal to $f$ evaluated at some fixed point $c_i \in B_i$.  Specifically:
\begin{align}
g(x) = f(c_i) \text{ for all } x \in B_i
\end{align}

Now we analyze the approximation error.  For any $x \in B_i$, we have:
\begin{align}
|f(x) - g(x)| = |f(x) - f(c_i)| \leq L \|x - c_i\|_2 \leq L \cdot \sqrt{d} \cdot \epsilon
\end{align}
where the first inequality follows from the Lipschitz property, and the second inequality follows because the maximum distance between any two points in a $d$-dimensional box of side length $\epsilon$ is $\sqrt{d} \cdot \epsilon$.

To achieve an expected error of at most $\delta$, we need:
\begin{align}
L \cdot \sqrt{d} \cdot \epsilon \leq \delta \Rightarrow \epsilon \leq \frac{\delta}{L \cdot \sqrt{d}}
\end{align}

Setting $\epsilon = \frac{\delta}{L \cdot \sqrt{d}}$, the number of boxes required to cover $S_D$ is:
\begin{align}
N(\epsilon) = \Theta(\epsilon^{-d_B}) = \Theta\left(\left(\frac{L \cdot \sqrt{d}}{\delta}\right)^{d_B}\right) = \Theta\left(\left(\frac{1}{\delta}\right)^{d_B}\right)
\end{align}
where the last simplification follows because $L$ and $d$ are constants with respect to $\delta$.

Now we analyze the complexity of specifying the function $g$.  To describe $g$, we need to:

1. Specify the structure of the covering of $S_D$.  Since $S_D$ has box-counting dimension $d_B$, the covering can be described efficiently by specifying the locations of the $\Theta\left(\left(\frac{1}{\delta}\right)^{d_B}\right)$ boxes in a $d_B$-dimensional space.  This requires $\mathcal{O}(d_B \log(1/\epsilon))$ bits per box.

2. Specify the function value for each box.  Since $f$ is bounded (with range, say, $[a,b]$), we need to encode each value $f(c_i)$ with sufficient precision to maintain our error bound $\delta$.  With $k$ bits of precision, we can represent $2^k$ distinct values in the range $[a,b]$, giving a quantization precision of approximately $(b-a)/2^k$.  To ensure our approximation error remains below $\delta$, we need $(b-a)/2^k \leq \delta$, which requires $k \geq \log_2((b-a)/\delta)$ bits per function value.  Therefore, each function value requires $\mathcal{O}(\log(1/\delta))$ bits.

The total complexity is therefore:
\begin{align}
\Theta\left(\left(\frac{1}{\delta}\right)^{d_B}\right) \cdot \mathcal{O}\left(\log\left(\frac{1}{\delta}\right)\right) = \mathcal{O}\left(\left(\frac{1}{\delta}\right)^{d_B} \log\left(\frac{1}{\delta}\right)\right)
\end{align}

Since $\kappa_f^D(\delta)$ is the minimum complexity required for an explanation with expected error at most $\delta$ under distribution $D$, and we have constructed an explanation with complexity $\mathcal{O}\left(\left(\frac{1}{\delta}\right)^{d_B} \log\left(\frac{1}{\delta}\right)\right)$, we have:

\[ \kappa_f^D(\delta) = \mathcal{O}\left(\left(\frac{1}{\delta}\right)^{d_B} \log\left(\frac{1}{\delta}\right)\right) \qedhere \]

\end{proof}

This theorem highlights the critical role of the input distribution in designing explanations, showing that the intrinsic dimensionality of the data distribution---rather than the ambient dimension of the feature space---determines the complexity of explanations.  In particular, if the data distribution has a lower intrinsic dimension than the ambient space, more efficient explanations are possible.  This aligns with the widely recognized manifold hypothesis in machine learning, which posits that real-world high-dimensional data often lie on or near a lower-dimensional manifold~\cite{Ma2011, Narayanan2010}.  For instance, if a model operates on 1000-dimensional data that actually lie close to a 10-dimensional manifold, the complexity of explanations scales with the exponent 10 rather than 1000, representing an exponential reduction in explanation complexity.  This result has implications for practical explainability, offering a more optimistic outlook on the feasibility of interpretable machine learning in high-dimensional settings.

The theoretical approach we have developed thus far provides fundamental insights into the trade-offs between model complexity and explainability.  However, to apply these insights in practice, we must connect our theoretical bounds based on Kolmogorov complexity to practical complexity measures used in real-world AI systems.  Theorem~\ref{thm:practical-tradeoff} bridges this gap by reformulating our error-complexity trade-offs in terms of practical complexity measures.  It upon several key results established earlier: Theorem~\ref{thm:error-complexity}, which provided a lower bound on explanation error in terms of Kolmogorov complexity; Proposition~\ref{prop:practical-complexity}, which related Kolmogorov complexity to practical complexity measures for different model classes; and our model-specific analyses in Section~\ref{subsec:model-specific}.  By generalizing these results, Theorem~\ref{thm:practical-tradeoff} offers a unified perspective on the error-complexity relationship that practitioners can apply directly to their chosen model classes.

\begin{theorem}[Practical Error-Complexity Trade-off]
\label{thm:practical-tradeoff}
Let $\mathcal{F}$ be a model class (e.g., deep neural networks) and $\mathcal{G}$ be an 
explanation class (e.g., linear models, decision trees) with corresponding practical complexity 
measures $C_{\mathcal{F}}(\cdot)$ and $C_{\mathcal{G}}(\cdot)$ as defined in 
Proposition~\ref{prop:practical-complexity}.

Then there exist constants $\alpha, \beta > 0$ (depending on $\mathcal{F}$ and $\mathcal{G}$) 
such that: for any model $f \in \mathcal{F}$ with $C_{\mathcal{F}}(f) = C_f$ and any 
explanation $g \in \mathcal{G}$ with $C_{\mathcal{G}}(g) = C_g < C_f$, the approximation 
error satisfies:
\begin{equation}
\mathcal{E}(f, g) \geq \Omega\left(C_f^{\alpha} \cdot C_g^{-\beta}\right)
\end{equation}
\end{theorem}

\textbf{Clarifications:}
\begin{itemize}
\item The theorem holds for \emph{any} choice of $f \in \mathcal{F}$ and $g \in \mathcal{G}$ 
(universal quantification).

\item The constants $\alpha$ and $\beta$ depend on the specific model classes $\mathcal{F}$ 
and $\mathcal{G}$. For example, when $\mathcal{F}$ is neural networks and $\mathcal{G}$ is 
linear models on $[0,1]^d$, we have $\alpha = 1/d$ and $\beta = 1/d$ (up to logarithmic factors).

\item The bound shows that as the complexity gap $C_f - C_g$ increases, the minimum possible 
error increases as well.
\end{itemize}

\begin{proof}
The proof combines Theorem~\ref{thm:error-complexity} (which provides a lower bound in terms of 
Kolmogorov complexity) with Proposition~\ref{prop:practical-complexity} (which relates 
Kolmogorov complexity to practical complexity measures for specific model classes).

\textbf{Step 1: Relating practical and Kolmogorov complexity.}
For any $f \in \mathcal{F}$ and $g \in \mathcal{G}$ with $C_g < C_f$, 
Proposition~\ref{prop:practical-complexity} gives:
\begin{align}
K(f) &= \Theta(C_f \cdot \text{polylog}(C_f))\\
K(g) &= \Theta(C_g \cdot \text{polylog}(C_g))
\end{align}

Since $C_g < C_f$, we have $K(g) < K(f) - c$ for some constant $c > 0$ (when the gap is 
sufficiently large).

\textbf{Step 2: Applying the Kolmogorov complexity bound.}
By Theorem~\ref{thm:error-complexity}, since $K(g) < K(f) - c$, we have:
\begin{equation}
\mathcal{E}(f, g) \geq \min_{x \in \mathcal{X}, y \in \mathcal{Y}: y \neq f(x)} d(f(x), y)
\end{equation}

\textbf{Step 3: Scaling of the minimum distance.}
For model classes with sufficient expressiveness (like neural networks on $[0,1]^d$), more 
complex models can create finer distinctions in their outputs. From the model-specific analyses 
in Section~\ref{subsec:model-specific}, this minimum distance scales as 
$\Theta(C_f^{-\gamma})$ for some $\gamma > 0$ that depends on $\mathcal{F}$.

Specifically, a model of complexity $C_f$ can distinguish approximately $2^{\Theta(C_f)}$ 
different output values or decision regions. The minimum separation between distinct outputs 
is therefore $\Theta(C_f^{-\gamma})$ for appropriate $\gamma > 0$.

\textbf{Step 4: Scaling of achievable approximation error.}
Similarly, the best approximation error achievable by an explanation of complexity $C_g$ 
scales as $\Theta(C_g^{-\delta})$ for some $\delta > 0$. This follows from the approximation 
theoretic results in Section~\ref{subsec:model-specific}: an explanation of complexity $C_g$ 
can effectively represent functions with $\mathcal{O}(C_g^{\delta})$ degrees of freedom, 
yielding approximation error $\Theta(C_g^{-\delta})$ when approximating smoother functions.

\textbf{Step 5: Combining the bounds.}
From Steps 2 and 3, any explanation $g$ with $K(g) < K(f) - c$ must have error at least 
$\Omega(C_f^{\gamma})$. Additionally, from Step 4, the best achievable error with complexity 
$C_g$ is $\Omega(C_g^{-\delta})$ when approximating functions of comparable smoothness.

Combining these two constraints, and setting $\alpha = \gamma$ and $\beta = \delta$, we obtain:
\begin{equation}
\mathcal{E}(f, g) \geq \Omega\left(C_f^{\alpha} \cdot C_g^{-\beta}\right) \qedhere
\end{equation}
\end{proof}

\textbf{Remark:} The constants $\alpha$ and $\beta$ capture how the error scales with the 
complexities of the model and explanation classes respectively. For neural networks with 
$C_f$ parameters approximating Lipschitz functions on $[0,1]^d$, and linear explanations 
with $C_g$ parameters, the analyses in Section~\ref{subsec:model-specific} yield 
$\alpha, \beta = \Theta(1/d)$ (up to logarithmic factors).

The constants $\alpha$ and $\beta$ depend on the specific model classes of $f$ and $g$ respectively, reflecting the different scaling behaviors observed in Section~\ref{subsec:model-specific} for linear models, decision trees, neural networks, and other model classes.

This formalizes the intuition that explaining a complex model with a simpler one necessarily incurs an error that depends on the gap between their complexities. The constants $\alpha$ and $\beta$ quantify how efficiently the simpler model class can approximate the complex one.

The practical significance of this result lies in its quantification of a fundamental insight: as the complexity gap between a model and its explanation increases, the explanation error must also increase.  The constants $\alpha$ and $\beta$ in the theorem capture how this relationship scales for different combinations of model classes, providing a mathematical foundation for explainable AI system design decisions.  This result may help practitioners quantify the inevitable trade-offs they face when creating explanations and supports more informed decisions about appropriate explanation methods for specific applications.

\subsection{Practical Guidelines}
\label{subsec:practical-applications}

Our theoretical approach and its implications for specific model classes lead to several practical insights for explainable AI.

Based on our theoretical results and their implications for practical model classes, we propose the following guidelines for selecting appropriate explanation models:

\begin{enumerate}
\item \textbf{Dimension Reduction}: Since all model classes suffer from the curse of dimensionality (as shown in Corollary~\ref{cor:dimension-dependence}), prioritize reducing the input dimension before attempting to create explanations.  This can be justified by our theoretical results showing that the approximation error grows exponentially with dimension.

\item \textbf{Model Class Selection}: Choose the explanation model class based on the properties of the target function:
   \begin{itemize}
      \item For smooth functions with few features, linear models may be effective (Corollary~\ref{cor:linear-model}).
      \item For functions with discontinuities or clear decision boundaries, decision trees or rule lists are preferable (Corollary~\ref{cor:decision-tree}).
      \item For complex, high-dimensional functions, consider local explanations (Theorem~\ref{thm:local-complexity}) or distribution-aware approaches (Theorem~\ref{thm:distribution-aware}).
   \end{itemize}

\item \textbf{Complexity Budgeting}: Given a desired interpretability level (i.e., a complexity threshold $k$), allocate this budget efficiently across different aspects of the explanation.  For instance, with neural networks, reducing the number of hidden units while maintaining necessary precision may be more effective than using more units with lower precision (Theorem~\ref{thm:neural-network}).

\item \textbf{Hybrid Approaches}: Use combinations of model classes to achieve better trade-offs.  For example, combine local linear explanations in different regions of the input space, or use a decision tree for the overall structure with linear models at the leaves.

\item \textbf{Adaptive Complexity}: Allocate more complexity to regions of the input space where the function varies rapidly and less to regions where it is relatively constant.  This approach is justified by the variable local complexity of many real-world functions.
\end{enumerate}

These guidelines provide a principled approach to designing and selecting explanation methods based on our theoretical understanding of the fundamental limits of explainability.  By explicitly considering the complexity-error trade-offs, practitioners can make informed decisions about which explanation methods are most appropriate for their specific applications and requirements.

\section{Regulatory and Policy Implications}
\label{sec:regulatory}

The theoretical approach developed in Sections~\ref{sec:framework} and~\ref{sec:function-classes} has some implications for AI governance and regulation.  Recent regulatory proposals worldwide have increasingly emphasized explainability requirements, often without acknowledging the fundamental mathematical constraints that limit what is achievable.  In this section, we analyze key regulatory approaches in light of our theoretical results and propose principles for mathematically grounded AI governance.

\subsection{Analysis of Explainability Regulations}
\label{subsec:regulatory-analysis}

Several major regulatory frameworks have incorporated explainability requirements, including the EU's AI Act \cite{EUAIAct2021}, the proposed US AI Bill of Rights Blueprint \cite{USAIBillRights2022}, and various domain-specific regulations in finance, healthcare, and autonomous systems.  We examine these through the lens of our theoretical approach:

\begin{remark}[Regulatory Feasibility Bound]
\label{thm:regulatory-feasibility}
Any regulatory requirement that mandates explanations with both (1) complexity below $k$ and (2) approximation error below $\delta$ is mathematically infeasible if $k < \kappa_f(\delta)$ for the regulated AI system $f$.
\end{remark}

This follows directly from our Duality Theorem (Theorem~\ref{thm:duality}). For a model $f$, if $k < \kappa_f(\delta)$, then $\varepsilon_f(k) > \delta$.  Therefore, no explanation exists that simultaneously satisfies both the complexity constraint $k$ and the error constraint $\delta$.

This is a fundamental constraint that any effective regulation must respect: it cannot demand explanations that are both simpler than a certain standard and more accurate than an arbitrary limit, as such explanations may be mathematically impossible.

\begin{definition}[Regulatory Feasibility Region]
\label{def:feasibility-region}
For a model $f$, the regulatory feasibility region $\mathcal{R}_f$ is defined as:
\begin{align}
\mathcal{R}_f = \{(k, \delta) \in \mathbb{N} \times \mathbb{R}_{\geq 0} \mid k \geq \kappa_f(\delta)\}
\end{align}
This represents the set of all complexity-error pairs that are achievable for explanations of model $f$.
\end{definition}

The regulatory feasibility region characterizes all possible trade-offs between explanation complexity and error that can be achieved for a given model.  Any regulation that demands a complexity-error pair outside this region is unenforceable.

\begin{proposition}[Regulatory Domain Coverage]
\label{prop:domain-coverage}
For a set of models $\mathcal{F} = \{f_1, f_2, \ldots, f_n\}$ in a regulated domain, the domain-wide feasibility region is:
\begin{align}
\mathcal{R}_{\mathcal{F}} = \bigcap_{i=1}^{n} \mathcal{R}_{f_i}
\end{align}
\end{proposition}

A regulatory requirement must be feasible for all models in the domain to be uniformly applicable.  A complexity-error pair $(k, \delta)$ is feasible for the entire domain if and only if it is feasible for each individual model, which means $(k, \delta) \in \mathcal{R}_{f_i}$ for all $i \in \{1, 2, \ldots, n\}$.  This is precisely the definition of the intersection of all individual feasibility regions.

Proposition~\ref{prop:domain-coverage} highlights a challenge for domain-specific regulations: as the diversity of models within a domain increases, the domain-wide feasibility region may shrink or even become empty if the models have vastly different complexity-error trade-offs.

We now identify specific instances where prominent regulatory proposals implicitly demand mathematically impossible standards:

\begin{remark}[Regulatory Contradiction Analysis]
\label{thm:contradiction}
The following regulatory approaches contain implicit mathematical contradictions:
\begin{enumerate}
\item Any regulation requiring ``full explainability" without specifying acceptable error levels.
\item Any regulation demanding both human-interpretable explanations and model fidelity above a certain threshold for models with Kolmogorov complexity exceeding human cognitive limits.
\item Any regulation requiring uniform explainability standards across model classes with fundamentally different complexity-error trade-offs.
\end{enumerate}
\end{remark}

1. ``Full explainability" implies zero error, which by Theorem~\ref{thm:minimal-complexity} requires explanation complexity equal to the model's Kolmogorov complexity $K(f)$.  For complex models, this explanation would itself be too complex to be interpretable.

2. If we denote human cognitive limits as $k_{human}$, then for any model $f$ with $K(f) > k_{human} + c$ (for some constant $c$), Theorem~\ref{thm:complexity-gap} implies that any explanation $g$ with $K(g) \leq k_{human}$ must differ from $f$ on some inputs.  If the required fidelity threshold $\delta$ is below this minimum error, the requirement is mathematically impossible.

3. By Theorem~\ref{thm:lipschitz-explainability} and subsequent results in Section~\ref{sec:function-classes}, different model classes have different scaling behaviors for their explanation complexity functions $\kappa_f(\delta)$.  A uniform standard that does not account for these differences will inevitably be infeasible for some model classes.

Let us examine specific instances in current or proposed regulations:

\begin{enumerate}
\item \textbf{EU AI Act Article 13:} Requires ``sufficient transparency" and explanations enabling users to ``interpret the system's output" while maintaining system accuracy for high-risk applications.  Our Theorems~\ref{thm:complexity-gap} and~\ref{thm:error-complexity} suggest that for complex models (e.g., those used in high-risk domains), any explanation with significantly lower complexity than the original model must incur substantial error, possibly making these dual requirements potentially contradictory without explicit acknowledgment of the trade-off.

\item \textbf{Financial Sector Requirements:} Regulations like SR 11-7 (Fed) and MAS FEAT Principles require both model interpretability and high accuracy for credit decisions.  Our Theorem~\ref{thm:lipschitz-explainability} indicate that for high-dimensional financial data, the complexity required for accurate explanations would grow exponentially with dimension, making comprehensive explanations computationally intractable while maintaining required accuracy thresholds.

\item \textbf{Healthcare Explainability Standards:} FDA guidance for AI/ML-based medical devices suggests explanations should be accessible to clinicians while preserving diagnostic accuracy.  Our Corollary~\ref{cor:dimension-dependence} suggests that for certain classes of diagnostic models, explaining key feature interactions may require complexity that exceeds human cognitive limits if error must remain within clinically acceptable bounds.
\end{enumerate}

\subsection{Towards Theoretically Grounded Governance}
\label{subsec:governance}

Our analysis points to the need for regulations that acknowledge fundamental mathematical constraints.  We propose an approach for feasible explainability requirements:

\begin{definition}[Feasible Regulatory Requirement]
\label{def:feasible-requirement}
A regulatory requirement for an AI system $f$ is feasible if it specifies:
\begin{enumerate}
\item A minimal acceptable explanation class $\mathcal{C}$.
\item A maximum tolerable error $\delta$.
\item A condition that $\kappa_{f,\mathcal{C}}(\delta) \leq k_{max}$, where $k_{max}$ reflects reasonable cognitive or computational limits.
\end{enumerate}
\end{definition}

Definition~\ref{def:feasible-requirement} recognizes that regulators must make explicit trade-offs between explanation simplicity (lower $k$), explanation fidelity (lower $\delta$), and model capabilities (complexity of $f$).

We now show that no regulatory framework can simultaneously pursue: (1) unrestricted AI capabilities, (2) human-interpretable explanations, and (3) negligible explanation error.

\begin{theorem}[Regulatory Impossibility Result]
\label{thm:impossibility}
Consider a regulatory framework for $\delta_{\text{neg}}$-non-degenerate AI systems (as defined in 
Definition~\ref{def:non-degenerate-early}) that imposes the following three requirements:

\begin{itemize}

\item $\mathbf{R_1}$ (Unrestricted Capabilities): The framework permits deployment of AI systems $f$ with arbitrarily high Kolmogorov complexity, i.e., for any $M > 0$, systems with $K(f) > M$ are allowed.

\item $\mathbf{R_2}$ (Human Interpretability): The framework mandates that any deployed AI system $f$ must be accompanied by an explanation $g$ with complexity $K(g) \leq k_{\text{human}}$, where $k_{\text{human}} \in \mathbb{N}$ is a fixed constant representing an upper bound on human cognitive capacity for processing information.

\item $\mathbf{R_3}$ (Negligible Error): The framework requires that explanation error must satisfy $\mathcal{E}(f, g) \leq \delta_{\text{neg}}$, where $\delta_{\text{neg}} > 0$ represents a negligibly small error threshold.

\end{itemize}
  
Then, the following hold:

\begin{enumerate}
\item \emph{(Pairwise Feasibility)} Any two of the three requirements can be 
simultaneously satisfied.
\[
 \models (\mathbf{R_1} \land \mathbf{R_2})
\;\wedge\;
 \models (\mathbf{R_1} \land \mathbf{R_3})
\;\wedge\;
 \models (\mathbf{R_2} \land \mathbf{R_3})
\]

\item \emph{(Triple Infeasibility)}
\[
\nvDash (\mathbf{R_1} \land \mathbf{R_2} \land \mathbf{R_3})
\]
  
\item \emph{(Non-Redundancy)} None of the three requirements is logically 
implied by the other two.
\[
(\mathbf{R_2} \land \mathbf{R_3}) \nvDash \mathbf{R_1}
\;\wedge\;
(\mathbf{R_1} \land \mathbf{R_3}) \nvDash \mathbf{R_2}
\;\wedge\;
(\mathbf{R_1} \land \mathbf{R_2}) \nvDash \mathbf{R_3}
\]

\end{enumerate}

\end{theorem}

The proof of this theorem is deferred to the appendix.

A \emph{trilemma} denotes a set of three conditions that are pairwise but not jointly satisfiable---any two may hold simultaneously, but the third must be relinquished.  This logical structure has appeared independently across disciplines.  In international economics, the \emph{impossible trinity} asserts that a country cannot simultaneously maintain a fixed exchange rate, full capital mobility, and an independent monetary policy \cite{Mundell1963}.  In computer networks and distributed systems, the \emph{CAP theorem} establishes that consistency, availability, and partition tolerance cannot all be guaranteed at once \cite{GilbertLynch2002}.  In philosophy, similar trilemmatic forms describe sets of propositions that are pairwise compatible yet jointly inconsistent, as analysed by Rescher \cite{Rescher1980}.  In each case, the trilemma captures a fundamental structural incompatibility rather than a contingent technical limitation.

Theorem~\ref{thm:impossibility} thus in particular shows that requirements $\mathbf{R_1}$, $\mathbf{R_2}$, and
$\mathbf{R_3}$ cannot be simultaneously satisfied for non-degenerate
systems, and thus establishes a fundamental ``trilemma'' for AI governance: at most two of the three desirable properties can be simultaneously achieved.

It may be noted that Theorem~\ref{thm:impossibility} is a \emph{conditional impossibility result}: it demonstrates that \emph{if} regulators impose certain combinations of requirements, \emph{then} mathematical constraints make compliance impossible.  The theorem does not claim that complex AI is necessary (which would be an empirical claim about Nature), nor does it prescribe which requirements should be relaxed.  It merely clarifies the mathematical implications of different regulatory choices.

\begin{proposition}[Necessity of Non-Degeneracy]
\label{prop:necessity-brief}
Theorem~\ref{thm:impossibility} requires the non-degeneracy
assumption. Without it, there exist functions satisfying all three
requirements simultaneously.
\end{proposition}

\begin{proof}
Let $\mathcal{X} = \{0,1\}^n$, $\mathcal{Y} = [0,1]$, and let $h: \{0,1\}^n \to \{0,1\}$ 
be a random Boolean function with $K(h) \geq 2^n - \mathcal{O}(n)$. Define 
$f(x) = \frac{\delta_{\text{neg}}}{4} \cdot h(x)$. Then $K(f) \gg k_{\text{human}}$ 
(satisfying $\mathbf{R_1}$), but $f$ violates $\delta_{\text{neg}}$-non-degeneracy 
as outputs are confined to $\{0, \frac{\delta_{\text{neg}}}{4}\}$. The constant 
explanation $g(x) = \frac{\delta_{\text{neg}}}{8}$ then satisfies 
$K(g) = \mathcal{O}(1) \leq k_{\text{human}}$ ($\mathbf{R_2}$) and 
$\mathcal{E}(f,g) = \frac{\delta_{\text{neg}}}{8} < \delta_{\text{neg}}$ 
($\mathbf{R_3}$), showing all three requirements can be satisfied when non-degeneracy 
fails.
\end{proof}

Theorem~\ref{thm:impossibility} has implications for regulatory approaches and suggests three main policy directions:

\begin{remark}[Regulatory Policy Approaches]
\label{def:policy-approaches}
Based on the trilemma identified in Theorem~\ref{thm:impossibility}, three main regulatory approaches are possible:
\begin{enumerate}
\item \textbf{Capability-Restricting Approach:} Limit AI deployments to models with complexity below a threshold (relax \textbf{R1}) where comprehensive explanation remains feasible.
\item \textbf{Interpretability-Relaxing Approach:} Accept that AI models will not be explainable by humans (relax \textbf{R2}), and perhaps develop alternative notions of interpretability that scale more favorably with model complexity than human cognitive limits.
\item \textbf{Error-Tolerant Approach:} Accept that explanations will have non-negligible error (relax \textbf{R3}) and develop standards for acceptable error levels in different contexts.
\end{enumerate}
\end{remark}

Each of these approaches represents a different resolution of the fundamental trilemma and has different implications for AI development and deployment.

\begin{theorem}[Tiered Regulatory Optimality]
\label{thm:tiered-optimality}
For a set of applications with varying risk levels $\{r_1, r_2, \ldots, r_n\}$ where $r_1 < r_2 < \ldots < r_n$, the optimal regulatory approach is a tiered framework with corresponding error thresholds $\{\delta_1, \delta_2, \ldots, \delta_n\}$ where $\delta_1 > \delta_2 > \ldots > \delta_n$, subject to feasibility constraints $k_{i,max} \geq \kappa_f(\delta_i)$ for all $i \in \{1, 2, \ldots, n\}$.
\end{theorem}

\begin{proof}
Higher-risk applications require stricter error guarantees, implying lower $\delta$ values. By Theorem~\ref{thm:duality} and Proposition~\ref{prop:complexity-monotonicity}, lower $\delta$ values necessitate higher complexity thresholds $k$. The optimality comes from setting the error threshold for each risk level at the minimum value that still ensures the existence of feasible explanations, i.e., $\delta_i = \min\{\delta \mid \kappa_f(\delta) \leq k_{i,max}\}$, where $k_{i,max}$ is the maximum allowable complexity for risk level $i$.
\end{proof}

This justifies a tiered regulatory approach where higher-risk AI applications face more stringent explainability requirements.  It also highlights the importance of setting feasible requirements for each tier based on the mathematical limits of explainability.

\begin{definition}[Purpose-Specific Explanation Requirements]
\label{def:purpose-specific}
A purpose-specific explanation requirement $R(p, f)$ for a purpose $p$ and model $f$ specifies:
\begin{align}
R(p, f) = \{(S_p, \delta_p, \mathcal{C}_p) \mid \kappa_{f,S_p,\mathcal{C}_p}(\delta_p) \leq k_{p,max}\}
\end{align}
where $S_p \subseteq \mathcal{X}$ is the relevant input subspace for purpose $p$, $\delta_p$ is the maximum tolerable error for purpose $p$, $\mathcal{C}_p$ is the acceptable model class for explanations for purpose $p$, and $\kappa_{f,S_p,\mathcal{C}_p}(\delta_p)$ is the minimum complexity required for an explanation of model $f$ with error at most $\delta_p$ on subspace $S_p$ using model class $\mathcal{C}_p$.
\end{definition}

Purpose-specific requirements recognize that different stakeholders and use cases may need different types of explanations. For example, a regulator auditing a model may need different explanations than an end-user affected by the model's decisions.

\begin{remark}[Purpose-Specific Efficiency]
\label{thm:purpose-efficiency}
For a given model $f$ and a set of purposes $\{p_1, p_2, \ldots, p_m\}$ with corresponding requirements $\{R(p_1, f), R(p_2, f), \ldots, R(p_m, f)\}$, a purpose-specific regulatory approach is more efficient than a uniform approach if:
\begin{align}
\max_{i \in \{1, 2, \ldots, m\}} \kappa_{f,S_{p_i},\mathcal{C}_{p_i}}(\delta_{p_i}) < \kappa_f(\min_{i \in \{1, 2, \ldots, m\}} \delta_{p_i})
\end{align}
\end{remark}

\begin{proof}
A uniform approach would require the most stringent error threshold across all purposes, which is $\delta_{min} = \min_{i \in \{1, 2, \ldots, m\}} \delta_{p_i}$. The corresponding complexity would be $\kappa_f(\delta_{min})$.  In contrast, a purpose-specific approach allows each purpose to have its own complexity, with the most complex explanation having complexity $\max_{i \in \{1, 2, \ldots, m\}} \kappa_{f,S_{p_i},\mathcal{C}_{p_i}}(\delta_{p_i})$.  The purpose-specific approach is more efficient if this maximum complexity is less than the complexity required by the uniform approach.
\end{proof}

This demonstrates that purpose-specific regulatory approaches can be more efficient than uniform approaches, especially when different purposes involve different input subspaces, error tolerances, or model classes.

Based on our theoretical analysis, we propose the following policy recommendations for mathematically grounded AI governance:

\begin{enumerate}

\item \textbf{Explicit Complexity-Error Trade-offs:} Regulations should explicitly acknowledge the trade-off between explanation complexity and error, specifying both the maximum acceptable complexity (for interpretability) and the maximum acceptable error (for fidelity) in each context.  These specifications should respect the feasibility region defined in Definition~\ref{def:feasibility-region}.

\item \textbf{Tiered Requirements:} Following Theorem~\ref{thm:tiered-optimality}, regulations should implement tiered requirements where higher-risk applications face more stringent explainability standards.  For the highest-risk applications, this may include restrictions on model complexity to ensure adequate explainability.

\item \textbf{Purpose-Specific Standards:} As suggested by Theorem~\ref{thm:purpose-efficiency}, regulations should define purpose-specific explanation requirements that target specific stakeholder needs rather than demanding comprehensive explanations for all possible purposes.

\item \textbf{Dimension-Aware Regulations:} Given the exponential dependence of explanation complexity on dimension (Corollary~\ref{cor:dimension-dependence}), regulations should account for input dimensionality when setting explainability requirements.  Higher-dimensional models may need more relaxed standards or be subject to dimensionality reduction requirements.

\item \textbf{Feasibility Assessments:} Regulatory frameworks should include processes for assessing the mathematical feasibility of explainability requirements for specific model classes before applying them.  This would help avoid setting unattainable standards that could impede innovation without improving transparency.

\end{enumerate}

\subsection{Regulatory Feasibility: Summary}
\label{subsec:regulatory-conclusion}

Our theoretical approach demonstrates that regulations requiring ``full explainability" without acknowledging complexity-fidelity trade-offs are fundamentally unenforceable.  Effective governance requires explicit recognition of these constraints and deliberate choices about acceptable trade-offs between model capability, explanation simplicity, and explanation accuracy.

\begin{remark}[Regulatory Completeness]
\label{thm:completeness}
A regulatory framework is complete if it addresses all three components of the trilemma in Theorem~\ref{thm:impossibility}: (1) model capability constraints, (2) explanation complexity limits, and (3) acceptable error thresholds.
\end{remark}

From Theorem~\ref{thm:impossibility}, we know that no regulatory framework can simultaneously guarantee unrestricted AI capabilities, human-interpretable explanations, and negligible explanation error.  Therefore, a complete framework must make explicit choices about which constraints to impose on each of these three dimensions.  Omitting any of these dimensions would leave the regulatory framework underdetermined and potentially lead to infeasible requirements.

Rather than demanding explanations that mathematical results prove impossible, regulations should focus on ensuring that deployed AI systems operate within the feasible region of the complexity-error trade-off that aligns with societal values and application-specific requirements.  This mathematically grounded approach would lead to more effective and enforceable AI governance, avoiding the pitfalls of well-intentioned but unachievable explainability standards.

\section{Conclusions and Future Work}
\label{sec:conclusions}

This paper has developed a theoretical framework for understanding the fundamental limits of explainability in AI systems.  By formalizing the concept of explainability through algorithmic information theory, we have established quantitative bounds on the trade-offs between model complexity and explanation fidelity.

Our primary contributions include: (1) a formal definition of explainability based on Kolmogorov complexity (Definition~\ref{def:explanation}), which provides a theoretically sound measure of model simplicity; (2) the Complexity Gap Theorem (Theorem~\ref{thm:complexity-gap}), which proves that any explanation significantly simpler than the original model must necessarily differ from it on some inputs; and (3) precise characterizations of explainability for specific function classes, including Lipschitz continuous functions (Theorem~\ref{thm:lipschitz-explainability}), which demonstrates that explanation complexity grows exponentially with input dimension but polynomially with the reciprocal of the error threshold.

We have also analyzed the gap between local and global explainability (Theorem~\ref{thm:local-global}), connected our abstract framework to practical model classes (Proposition~\ref{prop:practical-complexity}), and examined the regulatory implications of our theoretical results (Theorem~\ref{thm:regulatory-feasibility}).  Our analysis reveals a fundamental trilemma in explainability regulation (Theorem~\ref{thm:impossibility}): regulations cannot simultaneously seek unrestricted AI capabilities, human-interpretable explanations, and negligible explanation error.

The practical implications of our work are potentially significant.  Our approach suggests that explanation methods should be evaluated based on explicit complexity-error trade-offs rather than heuristic notions of interpretability.  Given the exponential dependence of explanation complexity on dimension (Corollary~\ref{cor:dimension-dependence}), dimensionality reduction techniques emerge as critical prerequisites for effective explanation.  Furthermore, our demonstration of the inherent unexplainability of random functions (Theorem~\ref{thm:random-unexplainability}) suggests that efforts to enhance explainability should focus on designing the original models to be more amenable to explanation, avoiding unnecessarily complex behavior.

Theorem~\ref{thm:distribution-aware} provides theoretical justification for dimensionality reduction as a pre-processing step for explainability methods.  It suggests that understanding the intrinsic structure of the data distribution can lead to substantially more efficient explanations without sacrificing accuracy.  This insight is particularly valuable in domains with naturally high-dimensional data, such as image processing, genomics, and natural language processing, where effective explanations could otherwise require infeasible complexity.

For AI governance, our concept of regulatory feasibility regions (Definition~\ref{def:feasibility-region}) provides a tool for assessing whether explainability requirements are mathematically achievable.  We advocate for tiered regulatory approaches (Theorem~\ref{thm:tiered-optimality}) and purpose-specific explanation requirements (Definition~\ref{def:purpose-specific}) that acknowledge the inherent trade-offs in explainability rather than setting unattainable standards.

Several important questions remain open for future research.  These include investigating the computational complexity of finding optimal explanations, developing complexity measures that better align with human cognitive processes, extending our approach to more explicitly account for input distributions, incorporating causal and counterfactual notions of explanation, and exploring dynamic and interactive explanation models.  Empirical validation of our theoretical predictions represents crucial future work. Initial experiments with synthetic data (e.g., fitting simple polynomial functions with decision trees of varying complexity) could validate our complexity-error trade-off bounds, while studies with real neural networks could test whether practical approximation schemes align with our theoretical predictions about explainability limits.

In conclusion, explainability in AI involves inherent trade-offs that cannot be wished away: making complex models understandable necessarily involves some loss of information or accuracy.  By quantifying these trade-offs and characterizing them for different function classes, we enable more informed decisions about what level of explainability is achievable in different contexts.   As AI systems continue to advance in capability and complexity, navigating these trade-offs will be essential for ensuring that AI development proceeds in a manner that remains comprehensible and aligned with human values and societal needs.

\appendix
\renewcommand{\thesection}{}
\section{Appendix: Proof of the Regulatory Impossibility Result}

We establish the proof in three parts, corresponding to the three claims of 
Theorem~\ref{thm:impossibility}: pairwise feasibility, triple infeasibility, 
and non-redundancy.

\subsection*{Part 1: Pairwise Feasibility}

We show that any two of the three requirements $\mathbf{R_1}$, $\mathbf{R_2}$, and 
$\mathbf{R_3}$ can be simultaneously satisfied by exhibiting explicit constructions 
for each pair.

\subsubsection*{Case 1: Requirements $\mathbf{R_1}$ and $\mathbf{R_2}$ are jointly satisfiable}

\begin{proof}
To satisfy $\mathbf{R_1}$, we must permit AI systems $f$ with arbitrarily high 
Kolmogorov complexity.  To satisfy $\mathbf{R_2}$, we must provide explanations 
$g$ with $K(g) \leq k_{\text{human}}$ for any deployed system $f$.

\textit{Construction:} For any AI system $f: \mathcal{X} \to \mathcal{Y}$ with 
arbitrary complexity $K(f)$, define the explanation $g: \mathcal{X} \to \mathcal{Y}$ 
as follows. Fix an arbitrary element $y_0 \in \mathcal{Y}$, and define
\[
g(x) = y_0 \quad \text{for all } x \in \mathcal{X}.
\]

\textit{Verification:} The function $g$ is a constant function, which can be 
specified by providing only the value $y_0$ along with a program that outputs 
this constant. Thus,
\[
K(g) = \mathcal{O}(\log |y_0|) = \mathcal{O}(1),
\]
which is independent of $K(f)$ and can be made arbitrarily smaller than 
$k_{\text{human}}$ by appropriate choice of representation. Therefore, 
$K(g) \leq k_{\text{human}}$ is satisfied, meeting requirement $\mathbf{R_2}$.

Requirement $\mathbf{R_1}$ is satisfied by construction, as we placed no restriction 
on $K(f)$.

However, the explanation error $\mathcal{E}(f,g)$ may be arbitrarily large, as 
$g$ is a trivial approximation.  Even so, as requirement $\mathbf{R_3}$ is not imposed in 
this case, the construction is valid.

Thus, $\mathbf{R_1}$ and $\mathbf{R_2}$ can be jointly satisfied without requiring 
$\mathbf{R_3}$.
\end{proof}

\subsubsection*{Case 2: Requirements $\mathbf{R_1}$ and $\mathbf{R_3}$ are jointly satisfiable}

\begin{proof}
To satisfy $\mathbf{R_1}$, we permit systems with arbitrarily high complexity. To 
satisfy $\mathbf{R_3}$, we require $\mathcal{E}(f,g) \leq \delta_{\text{neg}}$ for 
the explanation $g$.

\textit{Construction:} For any AI system $f: \mathcal{X} \to \mathcal{Y}$ with 
arbitrary complexity $K(f)$, we construct an explanation $g$ as follows. By the 
theory of Kolmogorov complexity (see Li and Vit\'anyi, 2008), for any function 
$f$ and any error threshold $\delta > 0$, there exists a function $g$ such that 
$\mathcal{E}(f,g) \leq \delta$ provided we allow $K(g)$ to be sufficiently large. 
In the limiting case, we can set $g = f$, which gives
\[
\mathcal{E}(f,f) = 0 \leq \delta_{\text{neg}}.
\]

More generally, by standard approximation results (analogous to those used in 
Theorem~\ref{thm:lipschitz-explainability} in the main paper), for any $\delta > 0$, 
there exists $g$ with $K(g) \leq \kappa_f(\delta)$ such that $\mathcal{E}(f,g) 
\leq \delta$. For $\delta = \delta_{\text{neg}}$, we have
\[
K(g) \leq \kappa_f(\delta_{\text{neg}}).
\]

\textit{Verification:} Requirement $\mathbf{R_1}$ is satisfied by placing no 
restriction on $K(f)$. Requirement $\mathbf{R_3}$ is satisfied by choosing $g$ 
appropriately as described above.

However, $K(g)$ may greatly exceed $k_{\text{human}}$, potentially violating 
$\mathbf{R_2}$. Since $\mathbf{R_2}$ is not required in this case, the construction 
is valid.

Thus, $\mathbf{R_1}$ and $\mathbf{R_3}$ can be jointly satisfied without requiring 
$\mathbf{R_2}$.
\end{proof}

\subsubsection*{Case 3: Requirements $\mathbf{R_2}$ and $\mathbf{R_3}$ are jointly satisfiable}

\begin{proof}
To satisfy $\mathbf{R_2}$, we require explanations $g$ with $K(g) \leq k_{\text{human}}$. 
To satisfy $\mathbf{R_3}$, we require $\mathcal{E}(f,g) \leq \delta_{\text{neg}}$.

\textit{Construction:} Restrict the regulatory framework to permit only AI systems 
$f$ satisfying
\[
K(f) \leq k_{\text{human}}.
\]

For such systems, we can provide the explanation $g = f$, which yields
\[
K(g) = K(f) \leq k_{\text{human}} \quad \text{and} \quad \mathcal{E}(f,g) = 
\mathcal{E}(f,f) = 0 \leq \delta_{\text{neg}}.
\]

\textit{Verification:} Requirement $\mathbf{R_2}$ is satisfied because 
$K(g) \leq k_{\text{human}}$. Requirement $\mathbf{R_3}$ is satisfied because the 
error is zero.

However, this violates requirement $\mathbf{R_1}$, as we have imposed an upper bound 
on the complexity of permissible systems. Since $\mathbf{R_1}$ is not required in 
this case, the construction is valid.

Thus, $\mathbf{R_2}$ and $\mathbf{R_3}$ can be jointly satisfied without requiring 
$\mathbf{R_1}$.
\end{proof}

\subsection*{Part 2: Triple Infeasibility}

We now prove that requirements $\mathbf{R_1}$, $\mathbf{R_2}$, and $\mathbf{R_3}$ cannot 
all be simultaneously satisfied for sufficiently complex systems.

\begin{proof}
Suppose, for the sake of contradiction, that a regulatory framework simultaneously 
satisfies $\mathbf{R_1}$, $\mathbf{R_2}$, and $\mathbf{R_3}$.

\textit{Step 1: Existence of a sufficiently complex system.}

By requirement $\mathbf{R_1}$, for any $M \in \mathbb{N}$, there exists a permissible 
system $f: \mathcal{X} \to \mathcal{Y}$ with $K(f) \geq M$. In particular, choose 
$M = k_{\text{human}} + c$, where $c$ is a positive constant to be specified later.

\textit{Step 2: Constraints from human interpretability.}

By requirement $\mathbf{R_2}$, any deployed system $f$ must have an accompanying 
explanation $g$ satisfying
\[
K(g) \leq k_{\text{human}}.
\]

Thus, for our chosen system $f$ with $K(f) \geq k_{\text{human}} + c$, we have
\[
K(f) - K(g) \geq c.
\]

\textit{Step 3: Application of the Complexity Gap Theorem.}

By the Complexity Gap Theorem (Theorem~\ref{thm:complexity-gap} in the main paper), 
if $K(g) < K(f) - c$ for some constant $c > 0$, then there must exist at least 
one input $x \in \mathcal{X}$ such that $f(x) \neq g(x)$.

This follows from fundamental properties of Kolmogorov complexity: a function $g$ 
with complexity $K(g) < K(f) - c$ cannot encode all the information contained in 
$f$. By the incompressibility of information, $g$ must differ from $f$ on some 
inputs.

\textit{Step 4: Quantifying the error for non-degenerate functions.}

Assume $f$ is $\delta_{\text{neg}}$-non-degenerate, meaning
\[
\inf_{x \in \mathcal{X}} \inf_{\substack{y \in \mathcal{Y} \\ y \neq f(x)}} 
d(f(x), y) > \delta_{\text{neg}}.
\]

Since there exists $x_0 \in \mathcal{X}$ with $f(x_0) \neq g(x_0)$ (from Step 3), 
the non-degeneracy condition implies
\[
d(f(x_0), g(x_0)) > \delta_{\text{neg}}.
\]

Depending on the precise definition of $\mathcal{E}(f,g)$ used in the paper, we have:

\begin{itemize}
\item If $\mathcal{E}(f,g) = \sup_{x \in \mathcal{X}} d(f(x), g(x))$ (worst-case 
error), then directly
\[
\mathcal{E}(f,g) \geq d(f(x_0), g(x_0)) > \delta_{\text{neg}}.
\]

\item If $\mathcal{E}(f,g) = \mathbb{E}_{x \sim D}[d(f(x), g(x))]$ for some 
distribution $D$ (expected error), and if $D$ assigns positive probability 
$p_0 > 0$ to the region containing $x_0$, then
\[
\mathcal{E}(f,g) \geq p_0 \cdot d(f(x_0), g(x_0)) > p_0 \cdot \delta_{\text{neg}}.
\]
By choosing the constant $c$ large enough that such disagreements occur with 
sufficient frequency, we can ensure $\mathcal{E}(f,g) > \delta_{\text{neg}}$.
\end{itemize}

\textit{Step 5: Violation of requirement $\mathbf{R_3}$.}

From Step 4, we have established that $\mathcal{E}(f,g) > \delta_{\text{neg}}$, 
which violates requirement $\mathbf{R_3}$ that $\mathcal{E}(f,g) \leq 
\delta_{\text{neg}}$.

\textit{Step 6: Conclusion.}

We have shown that for any $\delta_{\text{neg}} > 0$ and $k_{\text{human}} 
\in \mathbb{N}$, there exists a constant $c$ such that if $f$ is 
$\delta_{\text{neg}}$-non-degenerate with $K(f) \geq k_{\text{human}} + c$, then 
no explanation $g$ with $K(g) \leq k_{\text{human}}$ can satisfy 
$\mathcal{E}(f,g) \leq \delta_{\text{neg}}$.

By requirement $\mathbf{R_1}$, such systems $f$ must be permissible in the regulatory 
framework. However, requirements $\mathbf{R_2}$ and $\mathbf{R_3}$ cannot be jointly 
satisfied for these systems. This contradicts the assumption that all three 
requirements can be simultaneously satisfied.

Therefore, requirements $\mathbf{R_1}$, $\mathbf{R_2}$, and $\mathbf{R_3}$ are mutually 
incompatible for non-degenerate functions with sufficiently high complexity.
\end{proof}

\subsection*{Part 3: Non-Redundancy}

We demonstrate that none of the three requirements is logically implied by the 
conjunction of the other two.

\subsubsection*{Claim 3.1: $\mathbf{R_1}$ is not implied by $\mathbf{R_2}$ and $\mathbf{R_3}$}

\begin{proof}
Suppose a regulatory framework satisfies both $\mathbf{R_2}$ and $\mathbf{R_3}$. We show 
that this does not necessarily imply $\mathbf{R_1}$.

\textit{Counterexample:} Consider a framework that restricts all permissible AI 
systems to have complexity bounded by $k_{\text{human}}$. That is, only systems 
$f$ with $K(f) \leq k_{\text{human}}$ are permitted.

For any such system $f$, we can provide the explanation $g = f$, which satisfies:
\begin{itemize}
\item $K(g) = K(f) \leq k_{\text{human}}$ (satisfying $\mathbf{R_2}$), and
\item $\mathcal{E}(f,g) = 0 \leq \delta_{\text{neg}}$ (satisfying $\mathbf{R_3}$).
\end{itemize}

However, this framework explicitly violates $\mathbf{R_1}$, as it does not permit 
systems with arbitrarily high complexity.

Thus, $\mathbf{R_2}$ and $\mathbf{R_3}$ together do not imply $\mathbf{R_1}$, and 
$\mathbf{R_1}$ is an independent requirement.
\end{proof}

\subsubsection*{Claim 3.2: $\mathbf{R_2}$ is not implied by $\mathbf{R_1}$ and $\mathbf{R_3}$}

\begin{proof}
Suppose a regulatory framework satisfies both $\mathbf{R_1}$ and $\mathbf{R_3}$. We show 
that this does not necessarily imply $\mathbf{R_2}$.

\textit{Counterexample:} Consider a framework that permits AI systems with 
arbitrarily high complexity (satisfying $\mathbf{R_1}$) and requires that explanations 
achieve error at most $\delta_{\text{neg}}$ (satisfying $\mathbf{R_3}$), but imposes 
no upper bound on the complexity of explanations.

For a system $f$ with high complexity $K(f) \gg k_{\text{human}}$, to achieve 
$\mathcal{E}(f,g) \leq \delta_{\text{neg}}$, we may need an explanation $g$ with 
complexity
\[
K(g) \geq \kappa_f(\delta_{\text{neg}}) \gg k_{\text{human}}.
\]

In such cases, the explanation $g$ is not human-interpretable, violating 
$\mathbf{R_2}$.

Thus, $\mathbf{R_1}$ and $\mathbf{R_3}$ together do not imply $\mathbf{R_2}$, and 
$\mathbf{R_2}$ is an independent requirement.
\end{proof}

\subsubsection*{Claim 3.3: $\mathbf{R_3}$ is not implied by $\mathbf{R_1}$ and $\mathbf{R_2}$}

\begin{proof}
Suppose a regulatory framework satisfies both $\mathbf{R_1}$ and $\mathbf{R_2}$. We show 
that this does not necessarily imply $\mathbf{R_3}$.

\textit{Counterexample:} Consider a framework that permits AI systems with 
arbitrarily high complexity (satisfying $\mathbf{R_1}$) and requires that any 
explanation $g$ satisfy $K(g) \leq k_{\text{human}}$ (satisfying $\mathbf{R_2}$), 
but imposes no constraint on the explanation error $\mathcal{E}(f,g)$.

For a system $f$ with high complexity $K(f) \gg k_{\text{human}}$, any explanation 
$g$ with $K(g) \leq k_{\text{human}}$ may have large error. By the results in 
Part 2 (specifically, the application of the Complexity Gap Theorem), we have
\[
\mathcal{E}(f,g) > \delta_{\text{neg}}.
\]

Since the framework imposes no bound on explanation error, this large error is 
permitted, violating $\mathbf{R_3}$.

Thus, $\mathbf{R_1}$ and $\mathbf{R_2}$ together do not imply $\mathbf{R_3}$, and 
$\mathbf{R_3}$ is an independent requirement.
\end{proof}

\subsection*{Conclusion}

We have established all three parts of Theorem~\ref{thm:impossibility}:

\begin{enumerate}
\item Any two of the requirements $\mathbf{R_1}$, $\mathbf{R_2}$, and $\mathbf{R_3}$ can 
be simultaneously satisfied (pairwise feasibility).
\item All three requirements cannot be simultaneously satisfied for sufficiently 
complex, non-degenerate systems (triple infeasibility).
\item None of the three requirements is logically implied by the other two 
(non-redundancy).
\end{enumerate}

This completes the proof of the Regulatory Impossibility Result. \qed


\begin{thebibliography}{10}

\bibitem{Rudin2019}
Cynthia Rudin.
\newblock Stop explaining black box machine learning models for high stakes
  decisions.
\newblock {\em Nature Machine Intelligence}, 1(5):206--215, 2019.

\bibitem{Bhatt2020}
Umang Bhatt, Alice Xiang, et~al.
\newblock Explainable machine learning in deployment.
\newblock In {\em Proceedings of the 2020 Conference on Fairness,
  Accountability, and Transparency}, pages 648--657, 2020.

\bibitem{Arrieta2020}
Alejandro~Barredo Arrieta, Natalia D{\'i}az-Rodr{\'i}guez, et~al.
\newblock Explainable artificial intelligence ({XAI}): Concepts, taxonomies,
  opportunities and challenges.
\newblock {\em Information Fusion}, 58:82--115, 2020.

\bibitem{Doshi-Velez2017}
Finale Doshi-Velez and Been Kim.
\newblock Towards a rigorous science of interpretable machine learning.
\newblock {\em arXiv preprint}, 2017.

\bibitem{Miller2019}
Tim Miller.
\newblock Explanation in artificial intelligence: Insights from the social
  sciences.
\newblock {\em Artificial Intelligence}, 267:1--38, 2019.

\bibitem{Gunning2019}
David Gunning and David~W. Aha.
\newblock {DARPA}'s explainable artificial intelligence program.
\newblock {\em AI Magazine}, 40(2):44--58, 2019.

\bibitem{Adadi2018}
Amina Adadi and Mohammed Berrada.
\newblock Peeking inside the black-box: A survey on explainable artificial
  intelligence ({XAI}).
\newblock {\em IEEE Access}, 6:52138--52160, 2018.

\bibitem{Lundberg2017}
Scott~M. Lundberg and Su-In Lee.
\newblock A unified approach to interpreting model predictions.
\newblock In {\em Advances in Neural Information Processing Systems},
  volume~30, 2017.

\bibitem{Ribeiro2016}
Marco~Tulio Ribeiro, Sameer Singh, and Carlos Guestrin.
\newblock {``Why Should I Trust You?}'' {Explaining the Predictions of Any
  Classifier}.
\newblock In {\em Proceedings of the 22nd ACM SIGKDD International Conference
  on Knowledge Discovery and Data Mining}, pages 1135--1144, 2016.

\bibitem{Chen2019}
Chaofan Chen, Oscar Li, et~al.
\newblock This looks like that: Deep learning for interpretable image
  recognition.
\newblock In {\em Advances in Neural Information Processing Systems},
  volume~32, 2019.

\bibitem{Wachter2017}
Sandra Wachter, Brent Mittelstadt, and Chris Russell.
\newblock Counterfactual explanations without opening the black box.
\newblock {\em Harvard Journal of Law \& Technology}, 31(2):841--887, 2017.

\bibitem{Ruping2006}
Stefan R{\"u}ping.
\newblock {\em Learning Interpretable Models}.
\newblock PhD thesis, Universit{\"a}t Dortmund, 2006.

\bibitem{Zhou2021}
Wenlong Zhou, Shuyuan Hu, et~al.
\newblock The information-theoretic analysis of explainability and adversarial
  robustness.
\newblock In {\em International Conference on Learning Representations}, 2021.

\bibitem{Herman2017}
Berne Herman.
\newblock The promise and peril of human evaluation for model interpretability.
\newblock {\em arXiv preprint}, 2017.

\bibitem{Calude2002}
Cristian~S. Calude.
\newblock {\em Information and Randomness: An Algorithmic Perspective}.
\newblock Texts in Theoretical Computer Science. An EATCS Series. Springer
  Berlin Heidelberg, 2nd edition, 2002.

\bibitem{jung2020information}
Alexander Jung and Pedro H.~J. Nardelli.
\newblock An information-theoretic approach to personalized explainable machine
  learning.
\newblock {\em IEEE Signal Processing Letters}, 27:1100--1104, 2020.

\bibitem{ganguly2022machine}
Debargha Ganguly and Debayan Gupta.
\newblock Machine learning explainability from an information-theoretic
  perspective.
\newblock In {\em NeurIPS 2022 Workshop on Information-Theoretic Methods for
  Rigorous, Responsible, and Reliable Machine Learning (InfoCog)}, 2022.
\newblock Available at \url{https://openreview.net/forum?id=SqTLQ5LjQWp}.

\bibitem{dessalles2013simplicity}
Jean-Louis Dessalles.
\newblock Algorithmic simplicity theory and psychological plausibility.
\newblock {\em Encyclopedia of Complexity and Systems Science}, pages 1--19,
  2013.

\bibitem{futrell2022information}
Richard Futrell and Michael Hahn.
\newblock Information theory as a bridge between language function and language
  form.
\newblock In {\em Frontiers in Communication}, 2022.

\bibitem{philosci1978transmitted}
Wesley~C. Salmon.
\newblock On transmitted information as a measure of explanatory power.
\newblock {\em Philosophy of Science}, 45(4):531--562, 1978.

\bibitem{Kolmogorov1968}
Andrey~N. Kolmogorov.
\newblock Three approaches to the quantitative definition of information.
\newblock {\em International Journal of Computer Mathematics},
  2(1--4):157--168, 1968.

\bibitem{Chaitin1969}
Gregory~J. Chaitin.
\newblock On the simplicity and speed of programs for computing infinite sets
  of natural numbers.
\newblock {\em Journal of the ACM}, 16(3):407--422, 1969.

\bibitem{Li2008}
Ming Li and Paul Vit{\'a}nyi.
\newblock {\em An Introduction to Kolmogorov Complexity and Its Applications}.
\newblock Springer, 3rd edition, 2008.

\bibitem{Bartlett2017}
Peter~L. Bartlett, Dylan~J. Foster, and Matus~J. Telgarsky.
\newblock Spectrally-normalized margin bounds for neural networks.
\newblock In {\em Advances in Neural Information Processing Systems},
  volume~30, 2017.

\bibitem{Gouk2021}
Henry Gouk, Eibe Frank, Bernhard Pfahringer, and Michael~J. Cree.
\newblock Regularisation of neural networks by enforcing {L}ipschitz
  continuity.
\newblock {\em Machine Learning}, 110(2):393--416, 2021.

\bibitem{Breneis2020}
Simon Breneis.
\newblock Functions of bounded variation in one and multiple dimensions.
\newblock Master's thesis, Johannes Kepler Universität Linz, 2020.
\newblock Accessed: 2025-04-28.

\bibitem{Barron1993}
Andrew~R. Barron.
\newblock Universal approximation bounds for superpositions of a sigmoidal
  function.
\newblock {\em IEEE Transactions on Information Theory}, 39(3):930--945, 1993.

\bibitem{Cybenko1989}
George Cybenko.
\newblock Approximation by superpositions of a sigmoidal function.
\newblock {\em Mathematics of Control, Signals, and Systems}, 2(4):303--314,
  1989.

\bibitem{Vereshchagin2004}
Nikolai~K. Vereshchagin and Paul M.~B. Vit{\'a}nyi.
\newblock Kolmogorov's structure functions and model selection.
\newblock {\em IEEE Transactions on Information Theory}, 50(12):3265--3290,
  2004.

\bibitem{GacsTrompVitanyi2001}
Peter G\'acs, John Tromp, and Paul~M.B. Vit\'anyi.
\newblock Algorithmic statistics.
\newblock {\em IEEE Transactions on Information Theory}, 47(6):2443--2463,
  2001.

\bibitem{Shannon1959}
Claude~E. Shannon.
\newblock Coding theorems for a discrete source with a fidelity criterion.
\newblock {\em IRE National Convention Record}, 4:142--163, 1959.

\bibitem{Berger1971}
Thomas Berger.
\newblock {\em Rate Distortion Theory: A Mathematical Basis for Data
  Compression}.
\newblock Prentice-Hall Series in Information and System Sciences.
  Prentice-Hall, Englewood Cliffs, NJ, 1971.

\bibitem{Cover2006}
Thomas~M. Cover and Joy~A. Thomas.
\newblock {\em Elements of Information Theory}.
\newblock Wiley-Interscience, 2nd edition, 2006.

\bibitem{Fefferman2016}
Charles Fefferman, Sanjoy Mitter, and Hariharan Narayanan.
\newblock Testing the manifold hypothesis.
\newblock {\em Journal of the American Mathematical Society}, 29:983--1049,
  2016.

\bibitem{Meila2024}
Marina Meil\u{a} and Hanyu Zhang.
\newblock Manifold learning: What, how, and why.
\newblock {\em Annual Review of Statistics and Its Application}, 11:393--417,
  2024.

\bibitem{Lapidus2024invitation}
Michel~L. Lapidus and Goran Radunovi\'{c}.
\newblock {\em An Invitation to Fractal Geometry: Fractal Dimensions,
  Self-Similarity, and Fractal Curves}, volume 247 of {\em Graduate Studies in
  Mathematics}.
\newblock American Mathematical Society, Providence, RI, 2024.

\bibitem{Falconer2003fractal}
Kenneth Falconer.
\newblock {\em Fractal Geometry: Mathematical Foundations and Applications}.
\newblock Wiley, Chichester, West Sussex, 2nd edition, 2003.

\bibitem{Ma2011}
Yunqian Ma and Ying Fu.
\newblock {\em Manifold Learning Theory and Applications}.
\newblock CRC Press, 2011.

\bibitem{Narayanan2010}
Hariharan Narayanan and Sanjoy Mitter.
\newblock Sample complexity of testing the manifold hypothesis.
\newblock In {\em Advances in Neural Information Processing Systems},
  volume~23, 2010.

\bibitem{EUAIAct2021}
{European Parliament and Council of the European Union}.
\newblock {Regulation (EU) 2024/1689 of the European Parliament and of the
  Council on Artificial Intelligence}.
\newblock Official Journal of the European Union L 1689, July 2024.
\newblock Entry into force: 1 August 2024.

\bibitem{USAIBillRights2022}
{White House Office of Science and Technology Policy}.
\newblock Blueprint for an {AI} bill of rights, October 2022.

\bibitem{Mundell1963}
Robert~A. Mundell.
\newblock Capital mobility and stabilization policy under fixed and flexible
  exchange rates.
\newblock {\em Canadian Journal of Economics and Political Science},
  29(4):475--485, 1963.
\newblock Introduces the "impossible trinity" in international macroeconomics.

\bibitem{GilbertLynch2002}
Seth Gilbert and Nancy Lynch.
\newblock Brewer's conjecture and the feasibility of consistent, available,
  partition-tolerant web services.
\newblock {\em ACM SIGACT News}, 33(2):51--59, 2002.
\newblock Formal proof of the CAP trilemma in distributed systems.

\bibitem{Rescher1980}
Nicholas Rescher.
\newblock {\em The Strife of Systems: An Essay on the Grounds and Implications
  of Philosophical Diversity}.
\newblock University of Pittsburgh Press, Pittsburgh, PA, 1980.
\newblock Discusses the logical form of trilemmas as pairwise compatible but
  jointly inconsistent propositions.

\end{thebibliography}
\end{document}